\documentclass[10pt]{article}
\PassOptionsToPackage{table}{xcolor}
\usepackage[preprint]{tmlr}  

\usepackage{amsmath,amsfonts,bm}

\def\eqref#1{equation~\ref{#1}}

\def\1{\bm{1}}

\DeclareMathAlphabet{\mathsfit}{\encodingdefault}{\sfdefault}{m}{sl}
\SetMathAlphabet{\mathsfit}{bold}{\encodingdefault}{\sfdefault}{bx}{n}

\newcommand{\R}{\mathbb{R}}

\usepackage{hyperref}
\usepackage{url}
\usepackage{amsmath,amssymb,amsthm,mathtools,bm}
\usepackage{xcolor}
\usepackage{booktabs,multirow,graphicx}
\usepackage{caption}
\usepackage{placeins}
\usepackage{tikz}
\usetikzlibrary{arrows.meta,positioning,fit,calc}

\definecolor{oursbg}{RGB}{232,240,251}  
\renewcommand{\arraystretch}{1.12}

\newtheorem{theorem}{Theorem}
\newtheorem{proposition}[theorem]{Proposition}
\newtheorem{lemma}[theorem]{Lemma}
\newtheorem{corollary}[theorem]{Corollary}
\theoremstyle{remark}

\newcommand{\codename}{accompanying code}  

\title{Exact Phase-Type Point Processes for\\
Certified Measurement and Prediction of Event Streams}

\author{\name Sayak Dutta \email sayakdutta1002@gmail.com}

\def\month{MM}
\def\year{YYYY}
\def\openreview{\url{https://openreview.net/forum?id=XXXX}}

\newcommand{\amzShortShare}{31}
\newcommand{\amzBoxLL}{2.62}
\newcommand{\rhoPooled}{0.87}
\newcommand{\relaxPooled}{12}
\newcommand{\placeboCpi}{8.5}
\newcommand{\placeboNfp}{6.7}
\newcommand{\placeboPpi}{4.0}
\newcommand{\placeboFomc}{3.9}
\newcommand{\placeboRetail}{1.9}
\newcommand{\nPassPlacebo}{5}
\newcommand{\nKinds}{11}
\newcommand{\tDirMed}{1.1}
\newcommand{\tEchoLo}{16}\newcommand{\tEchoHi}{137}
\newcommand{\tNinetyLo}{5}\newcommand{\tNinetyHi}{11}
\newcommand{\ksMax}{0.044}
\newcommand{\edLo}{26}\newcommand{\edHi}{55}
\newcommand{\gapMaxPooled}{0.0054}
\newcommand{\nWindows}{732}
\newcommand{\nEventsScored}{7.3}

\newcommand{\nNewsWindows}{430}\newcommand{\nPlaceboWindows}{302}
\newcommand{\vOnePlacebo}{169}
\newcommand{\vOneKSmax}{0.58}
\newcommand{\vTwoPlacebo}{18.6}
\newcommand{\eptNoBackbone}{2.84}
\newcommand{\lOneSelected}{0.8}

\newcommand{\trackCmarginLo}{0.23}\newcommand{\trackCmarginHi}{0.45}
\newcommand{\driftCorrOursOne}{0.17}\newcommand{\driftCorrBaseOne}{0.06}
\newcommand{\bestActivityModel}{event-study OLS}
\newcommand{\nTestReleases}{163}
\newcommand{\eptxSeeds}{three to five seeds}
\newcommand{\rhoByYear}{0.91, 0.89, 0.85, 0.86, 0.85}
\newcommand{\scanSpeedTaxi}{2.8}

\newcommand{\scanSpeedAAPL}{2.9}
\newcommand{\scanSecAAPL}{4.5}
\newcommand{\eptxLLAmazon}{0.723}
\newcommand{\eptxSDAmazon}{0.002}
\newcommand{\eptxLLRetweet}{-5.676}
\newcommand{\eptxSDRetweet}{0.003}
\newcommand{\eptxLLTaxi}{0.517}
\newcommand{\eptxSDTaxi}{0.002}
\newcommand{\eptxLLTaobao}{1.361}
\newcommand{\eptxSDTaobao}{0.004}
\newcommand{\eptxLLStackOverflow}{-2.162}
\newcommand{\eptxSDStackOverflow}{0.001}

\newcommand{\pubsrc}{\citealp{chang2025deep,boyd2025hyper}}  
\begin{document}
\maketitle

\begin{abstract}
Temporal point processes are used both to \emph{measure} market mechanisms, for example how fast
prices absorb news and how much of the reaction is the market responding to itself, and to
\emph{predict} the next event. We show that one algebraic object, the phase-type (Erlang-chain)
state space, serves both purposes with exact likelihoods. For measurement, we build a multiscale
phase-type Hawkes process whose parameters all enter linearly. Its log-likelihood is concave, and
we compute the global maximum together with a Fenchel duality-gap certificate. A Markov embedding
gives the response to an exogenous shock, echoes included, in closed form, and this response is
stable if and only if the spectral radius of the branching matrix is below one. Applied to six
markets around 670 scheduled macroeconomic releases (2022--2026) with a built-in placebo test, the
model shows that the direct reaction to CPI, payrolls, PPI and the FOMC statement is absorbed in
about a second, while the market's own echoes stretch half of the response to
\tEchoLo--\tEchoHi{}\,s. For prediction, we introduce EPT-X, a neural point process that adds a
history-dependent mixture-of-Erlang renewal channel to a linear Hawkes backbone. Its intensity
and compensator remain in closed form, its states follow an affine recursion that trains with a
parallel scan, and it can start from the certified fit of the measurement model. EPT-X has the best test likelihood on
all five LOBSTER order-book streams and on Retweet and Taobao, and is level with the best published
model on StackOverflow. We also find that the EasyTPP implementation of IntensityFree clamps short
inter-event times before scoring, and that the gaps of the EasyTPP Amazon data are uniform on two
fixed intervals.

\end{abstract}

\section{Introduction}\label{sec:intro}

Price changes, order-book messages and trades arrive as streams of events in continuous time,
and temporal point processes are the standard way to model them. They are used for two rather
different purposes. In econometrics, Hawkes processes \citep{hawkes1971spectra} are used to
\emph{measure} market mechanisms. Their branching structure \citep{hawkes1974cluster} separates
exogenous events, such as
a reaction to news, from endogenous events triggered by earlier trading, and the fitted kernels
describe how quickly information is absorbed
\citep{hardiman2013critical,filimonov2012quantifying,rambaldi2015modeling}. In machine learning,
neural temporal point processes are used to \emph{predict} the next event, and they are compared
by their held-out likelihood on shared benchmarks
\citep{du2016recurrent,mei2017neural,zuo2020transformer,shchur2020intensity,xue2024easytpp,chang2025deep}.

Each use has a weakness. Measurement relies on non-convex fits whose optimality is unknown, and
on summaries that ignore part of the response. A common example is the ``half-life''
$\ln 2/\beta$ of an exponential news kernel with decay rate $\beta$, which describes the direct
reaction to news but not the further trading that this reaction sets off. On the prediction
side, most intensity-based neural models estimate the compensator $\int\lambda(t)\,\mathrm{d}t$
by Monte Carlo \citep{mei2017neural,zuo2020transformer,yang2022transformer,chang2025deep,boyd2025hyper}, so
their reported likelihoods contain sampling error. Models with an exact likelihood do exist.
FullyNN \citep{omi2019fully} models the cumulative hazard with a monotone network, and
IntensityFree \citep{shchur2020intensity} and DTPP \citep{panos2024decomposable} model the density
of the next gap directly. These models, however, give up the Hawkes structure that makes a model
interpretable \citep{shchur2021neural}.
We also find that reported likelihoods can depend on implementation details.

In this paper we show that a single algebraic object serves both purposes with exact
likelihoods. We call it a \emph{phase-type state space}: every kernel is a nonnegative
combination of Erlang densities whose rates lie on a logarithmic grid of time scales. Between
two events such a state evolves by an explicit matrix exponential, so intensities and their
integrals are available in closed form. The class is also rich, since phase-type densities are
dense in the set of densities on $\R_+$ \citep{asmussen1996fitting,lee2010modeling}.

\paragraph{Contributions.}
\begin{enumerate}
\item \textbf{Certified measurement (MSX).} We define a multiscale phase-type Hawkes process
with marked news kernels in which every parameter enters the intensity linearly. Its
log-likelihood is therefore concave. We compute the global maximum with accelerated EM followed
by a barrier Newton step, and each fit comes with a Fenchel duality gap that bounds its distance
from the optimum (Proposition~\ref{prop:gap}).
\item \textbf{Echo-corrected responses in closed form.} Embedding the process in a linear
system gives the mean response to a shock as $C e^{\mathcal A t}\bm z_0$. We prove that this
system is stable exactly when the spectral radius of the branching matrix is below one
(Theorem~\ref{thm:stab}), and we obtain absorption times, echo amplification and price-discovery
curves without simulation.
\item \textbf{Exact neural prediction (EPT-X).} We introduce a neural point process whose latent
state is phase-type. It combines a linear Hawkes backbone with a history-dependent renewal
channel, a mixture of Erlang distributions that also allows ``no event'', added as a competing
risk. The likelihood remains exact, the model contains the Hawkes backbone as a special case,
and the renewal channel can approximate any conditional distribution of the next gap
(Proposition~\ref{prop:eptx}). Its states follow an affine recursion, so training uses a parallel
scan (Lemma~\ref{lem:scan}), and it can start from the certified MSX fit, which gives it a
validation likelihood no lower than that of MSX (Proposition~\ref{prop:floor}).
\item \textbf{Benchmarks and two findings about the data.} EPT-X has the best likelihood on all
five LOBSTER order-book streams and on Retweet and Taobao, and it is level with the best published
model on StackOverflow. The EasyTPP implementation of IntensityFree replaces every inter-event time
below $10^{-5}$\,s by $10^{-5}$\,s before scoring; scored on the observed gaps, it loses its lead on
two order-book streams. The gaps of the EasyTPP Amazon data are uniform on two fixed intervals,
which limits what its time likelihood can show.
\item \textbf{Application.} We study six markets around 670 scheduled macroeconomic releases
between 2022 and 2026, with a placebo test built into the likelihood. The direct reaction to
CPI, payrolls, PPI and the FOMC statement is absorbed in about a second, but the market's
subsequent trading stretches half of the response to \tEchoLo--\tEchoHi{}\,s and ninety percent
of it to \tNinetyLo--\tNinetyHi{} minutes.
\end{enumerate}

Our work builds on linear-in-parameters kernel dictionaries \citep{bacry2015hawkes}, EM algorithms
for Hawkes processes \citep{veen2008estimation,zhou2013learning} and exact-likelihood neural point
processes \citep{omi2019fully,shchur2020intensity}, and on Erlang-kernel embeddings of Hawkes
processes \citep{duarte2019stability}. A single phase-type state space lets us connect
them: certified, closed-form measurement on one side, and on the other an exact neural model that
contains the linear Hawkes process and starts from it.

The paper follows this order. After the related work in Section~\ref{sec:related},
Section~\ref{sec:phasetype} introduces the phase-type state space. Section~\ref{sec:msx} builds the
measurement model on it and Section~\ref{sec:eptx} the prediction model. Section~\ref{sec:experiments}
evaluates both, Section~\ref{sec:application} applies the measurement model to macroeconomic
releases, and Section~\ref{sec:discussion} concludes.

\section{Related work}\label{sec:related}

\paragraph{Hawkes processes and their estimation.} Hawkes processes \citep{hawkes1971spectra} admit
a cluster (branching) representation \citep{hawkes1974cluster}, and their likelihood is that of a
general point process with conditional intensity \citep{daley2003introduction}. Exponential and
sum-of-exponential kernels are usually fitted by maximum likelihood \citep{bacry2015hawkes}, for
instance with the \texttt{tick} library \citep{bacry2018tick}. Other estimators exploit the
branching structure through EM \citep{veen2008estimation}, add low-rank and sparse penalties
\citep{zhou2013learning} or group-sparse Granger-causal penalties \citep{xu2016learning}, invert
the second-order statistics through a Wiener--Hopf system
\citep{bacry2012nonparametric,bacry2016first,bacry2016estimation}, match integrated cumulants
\citep{achab2017uncovering} or discretise the process into an INAR model
\citep{kirchner2017estimation}. The branching matrix also defines a graph of causal influence
\citep{embrechts2018hawkes}. None of these estimators returns a certificate of global optimality for
multiscale kernels with delayed peaks. Erlang kernels make a Hawkes process Markovian, which
\citet{duarte2019stability} use to study the stability and simulation of nonlinear Hawkes
processes; we use the same embedding for estimation, certification and closed-form responses.

\paragraph{Hawkes processes in finance.} Branching ratios close to one have been read as evidence of
strong endogeneity \citep{filimonov2012quantifying,hardiman2013critical}, and such nearly unstable
processes have their own limit theory \citep{jaisson2015limit}, although a mis-specified
intraday seasonality biases these estimates upward \citep{filimonov2015apparent,omi2017hawkes}.
Mutually exciting processes also describe contagion across markets \citep{aitsahalia2015modeling}
and lead--lag relations between assets \citep{huth2014high}. \citet{rambaldi2015modeling} add an
exogenous kernel for macroeconomic announcements to a Hawkes model of FX quotes, and
\citet{andersen2003micro} document how announcement surprises move prices. In this literature the
speed of absorption is read off the exogenous kernel. We compute the full response, echoes
included, in closed form, and test the news kernels against matched placebo windows within the
likelihood.

\paragraph{Neural temporal point processes.} \citet{shchur2021neural} and \citet{zhou2026advances}
review the field. RMTPP \citep{du2016recurrent} uses an exponential intensity between events, whose
integral is closed form. NHP \citep{mei2017neural}, SAHP \citep{zhang2020self}, THP
\citep{zuo2020transformer}, AttNHP \citep{yang2022transformer} and the state-space model S2P2
\citep{chang2025deep} compute the intensity with recurrent, attention or state-space encoders, and
estimate the compensator by Monte Carlo or quadrature; continuous-time models driven by neural
differential equations \citep{jia2019neural,chen2021neural} integrate the intensity numerically.
Simulation and prediction usually rely on thinning \citep{ogata1981lewis}. FullyNN
\citep{omi2019fully} models the cumulative hazard with a monotone network, and IntensityFree
\citep{shchur2020intensity} models the density of the next gap with a mixture of log-normals. DTPP
\citep{panos2024decomposable} pairs such a time model with an attention model for the marks. These
models have exact likelihoods but no Hawkes structure. Recent work also explores Mamba-style
encoders \citep{gao2024mamba}. Hyper Hawkes processes \citep{boyd2025hyper} let a network set the
dynamics of a latent linear Hawkes process at each event; a softplus read-out keeps the intensity
positive, so the compensator is again estimated by sampling. EasyTPP \citep{xue2024easytpp}
provides shared implementations, \citet{bosser2023predictive} study next-event prediction, and
HoTPP \citep{karpukhin2024hotpp} evaluates long-horizon forecasts.

EPT-X differs from these models in three ways. First, between events its intensity is a linear
read-out of a phase-type system that can be solved exactly, plus bounded hazard terms and a
renewal channel whose survival function is a mixture of Erlang survival functions. All of these
integrate in closed form, so the likelihood we train on and report is exact. Second, its backbone
is the linear Hawkes process itself, so a fitted model still has excitation kernels that can be
read off, and it can start from a certified MSX fit. Third, its states follow an affine recursion,
so it trains with a parallel scan, as linear state-space layers do \citep{smith2023simplified},
without sampling the compensator.

\paragraph{Phase-type distributions.} Phase-type distributions are the absorption times of finite
Markov chains, and they can be fitted by EM \citep{asmussen1996fitting}. Mixtures of Erlang
distributions are dense in the set of distributions on $(0,\infty)$
\citep{tijms1994stochastic,lee2010modeling}. We use this fact for the renewal channel, and the
Erlang-chain recursion for every state in the paper.

\section{Phase-type state spaces}\label{sec:phasetype}

Both models in this paper rest on one construction: kernels built from Erlang densities, and the
linear recursion that their states obey. We introduce it here once, and the following two sections
use it without change.

\paragraph{A dictionary of Erlang densities.} Fix a set of rates $\beta_1,\dots,\beta_K>0$ on a
logarithmic grid, so that each rate corresponds to a time scale $1/\beta_k$, and an Erlang order
$R\ge1$. For $k=1,\dots,K$ and $r=0,\dots,R-1$ we define the basis densities
\begin{equation}\label{eq:basis}
g_{kr}(t)=\beta_k\,\frac{(\beta_k t)^r}{r!}\,e^{-\beta_k t},\qquad t\ge0 .
\end{equation}
Each $g_{kr}$ is the density of a gamma (Erlang) distribution with shape $r+1$ and rate
$\beta_k$, so $\int_0^\infty g_{kr}(t)\,\mathrm{d}t=1$ and its mean is $(r+1)/\beta_k$. A kernel is a
nonnegative combination
\begin{equation}\label{eq:kernel}
\phi(t)=\sum_{k=1}^{K}\sum_{r=0}^{R-1}a_{kr}\,g_{kr}(t),\qquad a_{kr}\ge0 .
\end{equation}
Because each basis density integrates to one, the total mass of the kernel, which is its
branching ratio, is simply $\sum_{k,r}a_{kr}$. With $R=1$ the dictionary contains only
exponentials, and mixtures of exponentials can approximate any completely monotone kernel,
power laws included (Bernstein's theorem). Orders $r\ge1$ add kernels that start at zero and peak
later, which describe delayed reactions. More generally, phase-type densities are dense in the set
of densities on $\R_+$.

\paragraph{Exact recursion.} Every state in this paper is built from such Erlang chains: the
Hawkes excitation, the response to news, the neural latent state of EPT and the renewal channel of
EPT-X. For all of them, the evolution between two events and its integral are available in closed
form. The following lemma gives the formulas for one source and one rate.

\begin{lemma}[Erlang-chain recursion]\label{lem:recursion}
Let $t_1<t_2<\dots$ be the event times of one source, let $\beta>0$ be one rate of the
dictionary, and define the chain of states
\[
s_r(t)=\sum_{n:\,t_n<t} g_r(t-t_n),\qquad r=0,\dots,R-1,
\]
where $g_r$ is the basis density~(\ref{eq:basis}) with rate $\beta$. Write
$\bm s=(s_0,\dots,s_{R-1})^\top$, let $I$ be the $R\times R$ identity and let $B$ be the
$R\times R$ down-shift matrix ($B_{r,r-1}=1$ and all other entries zero).
\begin{enumerate}
\item Between events, $\dot{\bm s}=\beta(B-I)\,\bm s$; at an event, $s_0$ increases by $\beta$.
\item For a gap $\Delta\ge0$ with no event in $(t,t+\Delta)$, let $y=\beta\Delta$ and
$w_j=e^{-y}y^j/j!$. Then
\[
s_r(t+\Delta)=\sum_{j=0}^{r}w_j\,s_{r-j}(t),\qquad
\int_t^{t+\Delta}s_r(u)\,\mathrm{d}u=\frac{1}{\beta}\sum_{j=0}^{r}P(j+1,y)\,s_{r-j}(t),
\]
where $P(a,y)$ is the regularised lower incomplete gamma function.
\end{enumerate}
\end{lemma}

\begin{proof}
\emph{Step 1 (dynamics).} Differentiating~(\ref{eq:basis}) gives
\[
\frac{\mathrm{d}}{\mathrm{d}t}g_0(t)=-\beta\,g_0(t),\qquad
\frac{\mathrm{d}}{\mathrm{d}t}g_r(t)=\beta\bigl(g_{r-1}(t)-g_r(t)\bigr)\quad(r\ge1).
\]
Summing over past events gives $\dot s_0=-\beta s_0$ and $\dot s_r=\beta(s_{r-1}-s_r)$, which is
$\dot{\bm s}=\beta(B-I)\bm s$. Since $g_0(0)=\beta$ and $g_r(0)=0$ for $r\ge1$, a new event adds $\beta$ to
$s_0$ and leaves the other states unchanged.

\emph{Step 2 (evolution over a gap).} The solution is $\bm s(t+\Delta)=e^{\beta(B-I)\Delta}\bm s(t)$.
The matrices $B$ and $I$ commute and $B^R=0$, so
\[
e^{\beta(B-I)\Delta}=e^{-y}\,e^{yB}=e^{-y}\sum_{j=0}^{R-1}\frac{y^j}{j!}B^j .
\]
Because $B^j$ shifts a vector down by $j$ places, component $r$ of $e^{\beta(B-I)\Delta}\bm s(t)$
is $\sum_{j\le r}w_j s_{r-j}(t)$.

\emph{Step 3 (integral over a gap).} Integrating the weights over the gap,
\[
\int_0^{\Delta}e^{-\beta u}\frac{(\beta u)^j}{j!}\,\mathrm{d}u=\frac{1}{\beta}\,P(j+1,y),
\]
which gives the second formula.
\end{proof}

Evaluating a likelihood therefore costs $O(N d K R^2)$ operations for $N$ events in $d$
dimensions, and it involves no Monte Carlo or quadrature error. The next section uses this
recursion in a linear model for measurement, and Section~\ref{sec:eptx} in a neural model for
prediction.

\section{MSX: certified measurement with echo-corrected responses}\label{sec:msx}

We start with measurement. The aim is a Hawkes model whose fit can be certified and whose response
to a shock, echoes included, is available in closed form. We first state the model, then the
certificate, and then the embedding that gives the response.

\paragraph{Model.} We observe $d$ event types (for example upward and downward price moves of
several assets) inside observation windows $w$, each with a likelihood interval
$[t_0^w,t_1^w)$. Exogenous shocks $e$, here scheduled macroeconomic releases, occur at known times
$\tau_e$. Each shock has a type $c_e$ (CPI, payrolls, and so on) and a nonnegative mark vector
$\bm z_e=(1,z_e^+,z_e^-)$, where $z_e^+$ and $z_e^-$ are the positive and negative parts of the
standardised surprise. The MSX intensity of event type $i$ is
\begin{align}
\lambda_i(t)={}&\underbrace{c_{iw}+\sum_{m}\eta_{im}\,h_m\bigl(\mathrm{tod}(t)\bigr)}_{\text{baseline}}
+\underbrace{\sum_{j=1}^{d}\sum_{n:\,t_n^j<t}\sum_{k,r}a_{ijkr}\,g_{kr}(t-t_n^j)}_{\text{endogenous}}\nonumber\\
&+\underbrace{\sum_{e:\,\tau_e<t}\sum_{k,r,m}b_{i c_e k r m}\,z_{em}\,g^{\mathrm{exo}}_{kr}(t-\tau_e)}_{\text{response to news}}
+\underbrace{\sum_{e:\,\tau_e>t}\sum_{k}\alpha_{ic_ek}\,\beta^{\mathrm{ant}}_k e^{-\beta^{\mathrm{ant}}_k(\tau_e-t)}}_{\text{anticipation}} .
\label{eq:intensity}
\end{align}
The four terms have the following meaning. In the baseline, $c_{iw}$ is an intercept for window
$w$, and $h_m$ are periodic hat functions of the local time of day $\mathrm{tod}(t)$ with weights
$\eta_{im}$. In the endogenous term, $t_n^j$ is the $n$-th event of type $j$, the $g_{kr}$ are the
basis densities~(\ref{eq:basis}), and $a_{ijkr}$ is the weight with which an event of type $j$
excites type $i$ at time scale $k$ and order $r$. The response to news uses a second dictionary
$g^{\mathrm{exo}}_{kr}$ with its own rates, in which $b_{ickrm}$ weights mark component $m$ of a
release of type $c$. Finally, $\alpha_{ick}$ captures activity that builds up before a release of
type $c$, with rates $\beta^{\mathrm{ant}}_k$. All parameters are nonnegative. Every term is linear in its parameters, so we can write
$\lambda_i(t)=\bm x(t)^\top\bm\theta_i$, where $\bm\theta_i$ stacks the parameters of type $i$ and
$\bm x(t)\ge0$ stacks the corresponding known functions of time. The log-likelihood then
separates across event types:
\begin{equation}\label{eq:ll}
\ell_i(\bm\theta_i)=\sum_{n:\,u_n=i}\log\bigl(\bm x(t_n)^\top\bm\theta_i\bigr)-\bm\Phi^\top\bm\theta_i,
\qquad \bm\Phi=\sum_w\int_{t_0^w}^{t_1^w}\bm x(t)\,\mathrm{d}t ,
\end{equation}
where $u_n$ is the type of the $n$-th event. Every entry of $\bm x(t_n)$ and of $\bm\Phi$ follows
from Lemma~\ref{lem:recursion}.

\paragraph{Concavity and a certificate.} Problem~(\ref{eq:ll}) maximises a concave function over
the nonnegative orthant. We add an $\ell_1$ penalty $\bm\gamma^\top\bm\theta_i$ with weights
$\bm\gamma\ge0$ proportional to each feature's exposure; since $\bm\theta_i\ge0$, the penalty is
linear and simply adds $\bm\gamma$ to $\bm\Phi$. We maximise with the multiplicative EM update
\[
\theta_p\;\leftarrow\;\theta_p\,\frac{\sum_n x_{np}/\lambda_n}{\Phi_p+\gamma_p},
\]
where $x_{np}$ is feature $p$ at the $n$-th event of type $i$ and $\lambda_n$ the current intensity
there. We accelerate it with SQUAREM \citep{varadhan2008simple} and finish with a log-barrier
Newton method. The following bound tells us how far any fit is from the optimum.

\begin{proposition}[Duality-gap certificate]\label{prop:gap}
Fix an event type and drop the index $i$. Let $X$ be the matrix whose $n$-th row is
$\bm x(t_n)^\top$ for the $N$ events of this type, let $\bm\theta\ge0$ be any parameter vector with
intensities $\lambda_n=(X\bm\theta)_n>0$, and let $\ell^\star$ be the maximum of the penalised
log-likelihood $\ell(\bm\theta)=\sum_n\log\lambda_n-(\bm\Phi+\bm\gamma)^\top\bm\theta$. Define
\[
s=\max_p\frac{\sum_n x_{np}/\lambda_n}{\Phi_p+\gamma_p}.
\]
Then
\[
\ell^\star-\ell(\bm\theta)\;\le\;N\log s+\bigl(\bm\theta^\top(\bm\Phi+\bm\gamma)-N\bigr),
\]
and both terms on the right are zero at the optimum.
\end{proposition}
The proof, in Appendix~\ref{app:proofs}, applies Fenchel's inequality to each $\log\lambda_n$.
Every fit reported in this paper comes with this bound.

\paragraph{Markov embedding.} To study the response to a single release, we collect all
endogenous Erlang states $s_{jkr}$ into one vector $\bm u(t)$ and the states of the release's
news kernels into $\bm v(t)$. Let $F$ be the block-diagonal generator of the endogenous chains
(one block $\beta_k(B-I)$ per source and rate, as in Lemma~\ref{lem:recursion}) and $F_e$ the
generator of the news chains. Let $G_{\mathrm{in}}$ be the matrix that adds $\beta_k$ to the first
phase of every chain of source $j$ when an event of type $j$ occurs. Finally, let $\Theta$ hold the
endogenous weights $a_{ijkr}$ and $\Theta_e$ the news weights of this release, so that the excess
intensity is $\Theta\bm u(t)+\Theta_e\bm v(t)$.

Taking expectations, each event of type $j$ occurs at the mean rate $m_j(t)$ and adds
$G_{\mathrm{in}}\bm e_j$ to $\bm u$. The mean excess intensity $\bm m(t)$ caused by the release
therefore satisfies a linear system:
\begin{equation}\label{eq:embedding}
\dot{\bm z}=\mathcal A\bm z,\qquad
\mathcal A=\begin{pmatrix}F+G_{\mathrm{in}}\Theta & G_{\mathrm{in}}\Theta_e\\ 0 & F_e\end{pmatrix},\qquad
\bm m(t)=C\bm z(t),\qquad C=\begin{pmatrix}\Theta&\Theta_e\end{pmatrix},
\end{equation}
where $\bm z=(\bar{\bm u},\bm v)$ stacks the expected endogenous states and the news states, and
$\bm z(0)=\bm z_0=(\bm 0,\bm v_0)$ with $\bm v_0$ the news states just after the release.

\begin{theorem}[Stability]\label{thm:stab}
Let $G$ be the $d\times d$ branching matrix with entries $G_{ij}=\sum_{k,r}a_{ijkr}$, the expected
number of type-$i$ events triggered directly by one type-$j$ event, and let $\rho(G)$ be its
spectral radius. Then the matrix $F+G_{\mathrm{in}}\Theta$ is Hurwitz (all eigenvalues have negative
real part) if and only if $\rho(G)<1$.
\end{theorem}

\begin{proof}
\emph{Step 1.} Every block $\beta_k(B-I)$ of $F$ is lower triangular with diagonal $-\beta_k<0$
and nonnegative off-diagonal entries. Hence $F$ is Hurwitz with nonnegative off-diagonal entries,
and $-F$ is a nonsingular M-matrix. The matrix $P=G_{\mathrm{in}}\Theta$ is entrywise nonnegative.

\emph{Step 2.} For such a pair, $F+P$ is Hurwitz if and only if $\rho\bigl((-F)^{-1}P\bigr)<1$. This
is the regular-splitting theorem for M-matrices \citep[Ch.~3]{varga2000matrix}, applied to
$-(F+P)=(-F)-P$.

\emph{Step 3.} The products $AB$ and $BA$ have the same nonzero eigenvalues, so
\[
\rho\bigl((-F)^{-1}G_{\mathrm{in}}\Theta\bigr)=\rho\bigl(\Theta(-F)^{-1}G_{\mathrm{in}}\bigr).
\]

\emph{Step 4.} Column $j$ of $(-F)^{-1}G_{\mathrm{in}}=\int_0^\infty e^{Ft}G_{\mathrm{in}}\,\mathrm{d}t$
contains, for every chain of source $j$, the integral over time of the state left by a single
type-$j$ event. By Lemma~\ref{lem:recursion} this state is $g_{kr}(t)$, whose integral is one.
Hence
\[
\bigl(\Theta(-F)^{-1}G_{\mathrm{in}}\bigr)_{ij}=\sum_{k,r}a_{ijkr}\int_0^\infty g_{kr}(t)\,\mathrm{d}t
=\sum_{k,r}a_{ijkr}=G_{ij}.
\]
Combining Steps 2 to 4, $F+G_{\mathrm{in}}\Theta$ is Hurwitz if and only if $\rho(G)<1$.
\end{proof}

The condition $\rho(G)<1$ is the classical stationarity condition for linear Hawkes processes
\citep{hawkes1971spectra,bremaud1996stability}. What Theorem~\ref{thm:stab} adds is that the finite
Erlang embedding inherits it exactly, so the closed-form response, absorption times and
amplification below are valid precisely when the process is stationary.

\begin{corollary}[Closed-form absorption]\label{cor:abs}
Assume $\rho(G)<1$. Then $\mathcal A$ is Hurwitz, because it is block triangular and $F_e$ is
Hurwitz, and the following quantities are available in closed form.
\begin{enumerate}
\item The mean excess intensity after a release, echoes included: $\bm m(t)=Ce^{\mathcal At}\bm z_0$.
\item Its total, the expected number of extra events: $M=\int_0^\infty\bm m(t)\,\mathrm{d}t=-C\mathcal A^{-1}\bm z_0$.
\item The fraction of the response absorbed by time $t$, summed over event types:
\[
\Pi(t)=\frac{\bm 1^\top C\mathcal A^{-1}\bigl(e^{\mathcal At}-I\bigr)\bm z_0}{-\bm 1^\top C\mathcal A^{-1}\bm z_0}.
\]
Its quantiles $t_q=\Pi^{-1}(q)$ are the echo-corrected absorption times; we report $t_{50}$ and $t_{90}$.
\item The amplification $\kappa=M/M_{\mathrm{direct}}$, where $M_{\mathrm{direct}}=-\Theta_eF_e^{-1}\bm v_0$
is the total of the direct news response $\bm m_{\mathrm{direct}}(t)=\Theta_e e^{F_et}\bm v_0$ alone.
\item For an asset whose price is observed through upward and downward crossings of size $\delta$
(types $u$ and $v$), the expected price change up to time $t$:
$\delta\int_0^t\bigl(m_u(s)-m_v(s)\bigr)\,\mathrm{d}s$.
\end{enumerate}
\end{corollary}

The usual ``news half-life'' $\ln 2/\beta$ describes only $\bm m_{\mathrm{direct}}$. When $\rho(G)$
is close to one, the echoes dominate the response, and Section~\ref{sec:application} shows that
the two can differ by one to two orders of magnitude.

\section{EPT-X: exact neural prediction}\label{sec:eptx}

MSX is linear by design, which is what makes it certifiable. For prediction we keep its state space
and its exact likelihood, and let a network adapt the state to the history. Events arrive at times $t_1<t_2<\dots$ with marks
$y_n\in\{1,\dots,M\}$, and we write $\Delta_n=t_n-t_{n-1}$ for the gaps. For a time $t$ after the
$n$-th event, $s=t-t_n$ denotes the time elapsed since that event.

\begin{figure}[t]\centering
\begin{tikzpicture}[
    font=\small,
    >={Stealth[length=2.2mm]},
    box/.style={draw, thick, minimum height=8mm, minimum width=27mm, align=center, inner sep=3pt},
    state/.style={box, fill=black!6},
    lab/.style={font=\footnotesize, inner sep=1.5pt},
  ]
  \node (in) at (0,0) {$(\Delta_n,\,y_n)$};
  \node[box, minimum width=14mm] (gru) at (2.1,0) {GRU};
  \draw[->] (in) -- (gru);

  \node[state] (hk) at (7.0, 2.4)  {Hawkes state $\bm\xi^{\mathrm H}$};
  \node[state] (xi) at (7.0, 0.95) {Erlang state $\bm\xi$};
  \node[box]   (hz) at (7.0,-0.5)  {hazard terms};
  \node[box]   (rn) at (7.0,-1.95) {Erlang mixture};

  \coordinate (split) at (3.6,0);
  \draw[thick] (gru) -- node[lab, above] {$\bm h_n$} (split);
  \draw[->] (split) |- node[lab, above, pos=0.68] {gate, jump} (xi);
  \draw[->] (split) |- node[lab, above, pos=0.68] {$\mu, g, \alpha, \kappa$} (hz);
  \draw[->] (split) |- node[lab, above, pos=0.68] {$\pi, \omega$} (rn);

  \draw[->] (in.north) |- node[lab, above, pos=0.72] {$+\beta$ on mark $y_n$} (hk);

  \node[draw, dashed, rounded corners=2pt, inner sep=5pt, fit=(hk)(xi)(hz)(rn)] (cf) {};
  \node[lab, below=1pt of cf.south] {closed form for $t\in(t_n,t_{n+1}]$};

  \node[draw, thick, minimum width=5mm, minimum height=50mm, inner sep=0pt] (sum) at (11.0,0.225) {$\Sigma$};
  \draw[->] (hk.east) -- node[lab, above] {$\Theta\bm\xi^{\mathrm H}$} (hk.east -| sum.west);
  \draw[->] (xi.east) -- node[lab, above] {$C\bm\xi$} (xi.east -| sum.west);
  \draw[->] (hz.east) -- (hz.east -| sum.west);
  \draw[->] (rn.east) -- node[lab, above] {$p/S_{\mathrm{ren}}$} (rn.east -| sum.west);

  \node[box, dashed, minimum width=15mm] (mk) at (13.1,0.225) {mark law};
  \draw[->] (sum) -- node[lab, above] {$\lambda$} (mk);
  \node (lam) at (15.0,0.225) {$\lambda_y(t)$};
  \draw[->] (mk) -- (lam);
\end{tikzpicture}
\caption{EPT-X at one event. The GRU summary $\bm h_n$ sets the jump of the neural Erlang state,
the hazard parameters and the renewal weights, while the Hawkes state is driven by the event mark
alone. The grey boxes are the states carried from one event to the next. Until the next event,
all four terms evolve in closed form. The dashed box is the optional residual mark law, which
splits the total intensity across marks without changing it.}\label{fig:eptx}
\end{figure}
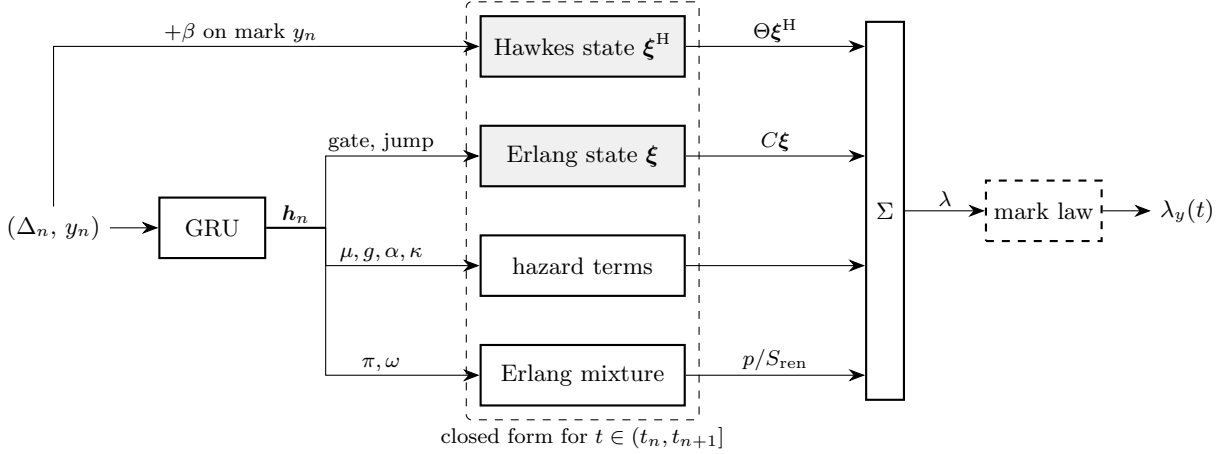

\paragraph{EPT: a neural phase-type state.} The latent state $\bm\xi(t)\in\R_{\ge0}^P$ consists of
$L$ channels of Erlang chains of order $R$, each with its own learnable rate. Between events it
evolves exactly as in Lemma~\ref{lem:recursion}. At each event, a recurrent network summarises the
history: a GRU \citep{cho2014learning} receives an embedding of the mark $y_n$ and log transforms of
the gap $\Delta_n$ and produces a hidden vector $\bm h_n$. In EPT-X the gap enters through its
standardised logarithm, which keeps nanosecond-scale gaps distinguishable from zero. The state
then jumps to
\[
\bm\xi(t_n^+)=\sigma(W_g\bm h_n)\odot\bm\xi(t_n^-)+\mathrm{softplus}(W_\xi\bm h_n),
\]
where $\sigma$ is the logistic function, $\odot$ the elementwise product and $W_g,W_\xi$ learned
matrices. The first term lets the network forget part of the old state; the second adds a
nonnegative jump. For $t\in(t_n,t_{n+1}]$ the intensity of mark $y$ is
\begin{equation}\label{eq:ept}
\lambda^{\mathrm{EPT}}_y(t)=\mu_y+\sum_{p=1}^{P}C_{yp}\,\xi_p(t)
+g_{1y}\,e^{-\alpha s}+g_{2y}\bigl(1-e^{-\kappa s}\bigr)
+\sum_{q}\Theta_{yq}\,\xi^{\mathrm{H}}_q(t).
\end{equation}
The quantities $\mu_y$, $C_{yp}$, $g_{1y}$, $g_{2y}$, $\alpha$ and $\kappa$ are nonnegative outputs
of the network computed from $\bm h_n$, so they stay fixed until the next event. The two hazard
terms let the intensity decay or rise with $s$, but stay bounded. An unbounded Gompertz term would
overflow on heavy-tailed gaps. The last term is a \emph{linear multivariate phase-type Hawkes
process}: $\bm\xi^{\mathrm{H}}(t)$ collects Erlang chains on a fixed logarithmic grid of rates, one
group per source mark, and every event of mark $y$ adds $\beta_k$ to the first phase of the chains
of source $y$. Its weights $\Theta_{yq}\ge0$ are learned. This term is a residual path: when the
neural read-outs are zero, EPT \emph{is} the Hawkes process of Section~\ref{sec:msx}, so the network
only has to learn departures from linear excitation.

\paragraph{The renewal channel.} Order-book and social streams contain cascades, several events in
a very short interval followed by a pause. A linear Hawkes intensity describes excitation
\emph{between} cascades well, but it has no direct handle on the distribution of gaps
\emph{within} them. EPT-X therefore adds a second channel that models the next gap directly.

We fix $Q$ Erlang ``atoms'' indexed by $q$, each with an order $R_q$ and a rate $\beta_q$. Their means
$R_q/\beta_q$ lie on a logarithmic grid that spans the time scales of the data, and their orders
come from $\{1,4,16\}$, optionally up to $1024$. An atom of order $R_q$ has coefficient of variation
$R_q^{-1/2}$, so high orders give sharp peaks in log-time. From $\bm h_n$ the network outputs
three things. The first is a set of weights $\pi_1,\dots,\pi_Q,\pi_\infty\ge0$ that sum to one,
where $\pi_\infty$ is the weight of a \emph{defect} atom meaning ``the renewal channel produces no
event''. The second is, for each time scale, a distribution $\omega_{y\mid q}$ over the marks, so
that the mark of a burst may depend on how quickly it follows the last event. The third, which is
optional, is a bounded shift of each log-rate, which lets an atom move within its grid cell. With
these outputs, the renewal channel's joint density of the next gap $s$ and mark $y$, and its survival function,
are
\begin{equation}\label{eq:renewal}
p(s,y)=\sum_{q=1}^{Q}\pi_q\,\omega_{y\mid q}\,
\frac{\beta_q^{R_q}\,s^{R_q-1}e^{-\beta_q s}}{(R_q-1)!},\qquad
S_{\mathrm{ren}}(s)=\pi_\infty+\sum_{q=1}^{Q}\pi_q\,e^{-\beta_q s}\sum_{i=0}^{R_q-1}\frac{(\beta_q s)^i}{i!}.
\end{equation}
EPT-X adds the corresponding hazard to EPT:
\[
\lambda_y(t)=\lambda^{\mathrm{EPT}}_y(t)+\frac{p(s,y)}{S_{\mathrm{ren}}(s)} .
\]
Figure~\ref{fig:eptx} shows how the parts fit together. The network only acts at event times;
between events every term evolves in closed form.
\begin{proposition}\label{prop:eptx}
Let $\Lambda^{\mathrm{EPT}}(\Delta)=\int_{t_n}^{t_n+\Delta}\sum_y\lambda^{\mathrm{EPT}}_y(t)\,\mathrm{d}t$
be the compensator of EPT over a gap of length $\Delta$.
\begin{enumerate}
\item The compensator of EPT-X over the same gap is $\Lambda^{\mathrm{EPT}}(\Delta)-\log S_{\mathrm{ren}}(\Delta)$,
and both terms are in closed form.
\item If $\pi_\infty=1$, EPT-X coincides with EPT, which in turn contains the linear phase-type
Hawkes process.
\item Given $\bm h_n$, the gap produced by the renewal channel is the absorption time of a finite
continuous-time Markov chain, so its survival function, quantiles and moments are closed-form
matrix expressions.
\item As the grid of atoms is refined, the renewal channel can approximate any conditional
distribution of the next gap.
\end{enumerate}
\end{proposition}

\begin{proof}
\emph{Part 1.} The intensity of EPT-X is a sum of two intensities, so its compensator is the sum of
the two compensators. For the renewal channel, $\sum_y p(s,y)=-\frac{\mathrm{d}}{\mathrm{d}s}S_{\mathrm{ren}}(s)$,
because each Erlang density is the negative derivative of its survival function. Hence
\[
\int_0^{\Delta}\frac{\sum_y p(s,y)}{S_{\mathrm{ren}}(s)}\,\mathrm{d}s
=-\int_0^{\Delta}\frac{\mathrm{d}}{\mathrm{d}s}\log S_{\mathrm{ren}}(s)\,\mathrm{d}s
=-\log S_{\mathrm{ren}}(\Delta),
\]
using $S_{\mathrm{ren}}(0)=1$. The EPT part is closed form by Lemma~\ref{lem:recursion} and by direct
integration of the two exponential hazard terms.

\emph{Part 2.} If $\pi_\infty=1$, all other weights are zero. Then $p\equiv0$ and $S_{\mathrm{ren}}\equiv1$,
so the renewal hazard vanishes. Setting the neural read-outs of EPT to zero leaves the Hawkes term.

\emph{Part 3.} An Erlang distribution of order $R$ is the absorption time of a chain of $R$
exponential phases. A mixture of such chains, entered at random according to the weights
$\pi_q$, together with an absorbing ``never'' state of weight $\pi_\infty$, is again a finite
continuous-time Markov chain with one absorbing state.

\emph{Part 4.} Mixtures of Erlang distributions are dense in the distributions on $(0,\infty)$
\citep{tijms1994stochastic,lee2010modeling}. The renewal channel is such a mixture, with weights
that depend on the history. This is the same expressiveness as a conditional mixture density
\citep{shchur2020intensity}, with the difference that the intensity and compensator stay exact.
\end{proof}

\paragraph{Training and prediction.} We maximise the exact log-likelihood
$\sum_n\log\lambda_{y_n}(t_n)-\Lambda$, where $\Lambda$ is the total compensator, so the gradients
contain no sampling noise. For large orders we evaluate $S_{\mathrm{ren}}$ through the regularised
upper incomplete gamma function, and we take care that its gradient stays finite at zero gaps. The
log-likelihood splits exactly into a time term, $\log\lambda(t_n)-\Lambda$ with
$\lambda=\sum_y\lambda_y$, and a mark term, $\log\bigl(\lambda_{y_n}(t_n)/\lambda(t_n)\bigr)$, which we
report separately. Predictions of the next event are Bayes-optimal functionals of the exact
density: the expected gap is $\int_0^\infty S(s)\,\mathrm{d}s$, where $S$ is the survival function
of the full model, and the predicted mark is $\arg\max_y\int_0^\infty\lambda_y(s)S(s)\,\mathrm{d}s$.
Both integrals are computed by one-dimensional quadrature on a logarithmic grid.

The rest of this section describes how the model is evaluated in parallel, two optional
components, and a way to start training from the certified fit of Section~\ref{sec:msx}.

\paragraph{Parallel evaluation.} The encoder reads only the observed gaps and marks, never the
latent state. All summaries $\bm h_1,\dots,\bm h_N$ can therefore be computed first, in one call to
the recurrent network. Given them, the state updates are affine.

\begin{lemma}\label{lem:scan}
Let $E(\Delta)$ be the block-diagonal transfer matrix of the Erlang chains over a gap $\Delta$
(Lemma~\ref{lem:recursion}), and let $D_n=\mathrm{diag}\bigl(\sigma(W_g\bm h_n)\bigr)E(\Delta_n)$ and
$\bm b_n=\mathrm{softplus}(W_\xi\bm h_n)$. Then $\bm\xi(t_n^+)=D_n\bm\xi(t_{n-1}^+)+\bm b_n$, and all
post-event states follow from a prefix scan in $O(\log N)$ parallel steps. Every entry of every
$D_n$ lies in $[0,1]$.
\end{lemma}

\begin{proof}
The first claim restates the jump rule after evolving the state over the gap by Lemma~\ref{lem:recursion}.
Affine maps compose associatively, $(D_2,\bm b_2)\circ(D_1,\bm b_1)=(D_2D_1,\,D_2\bm b_1+\bm b_2)$,
so a parallel prefix scan applies. The entries of $E(\Delta)$ are Poisson probabilities and the gates
lie in $(0,1)$, so the entries of $D_n$ lie in $[0,1]$ and the scan cannot overflow.
\end{proof}

The Hawkes state follows the same scan without gates. Once the post-event states are known, the
intensity and compensator of every gap are evaluated independently. The likelihood is the same as
with the event loop; only the order of computation changes.

\paragraph{Residual mark law.} The compensator depends only on the total intensity
$\lambda=\sum_y\lambda_y$. The split of $\lambda$ across marks can therefore be made more flexible
without losing exactness. We set
\[
\tilde\lambda_y(t)=\lambda(t)\,\mathrm{softmax}_y\bigl(\log\lambda_y(t)+m_y(\bm h_n,s)\bigr),
\]
where $m$ is a small network of the history summary and the elapsed time $s$. The total intensity,
and hence the compensator, is unchanged. With $m=0$ the marks are those of EPT-X, so this is a
residual correction. Separate models for times and marks were used before by
\citet{shchur2020intensity} and \citet{panos2024decomposable}; here the time model keeps its Hawkes
structure.

\paragraph{Data-adaptive atoms.} A logarithmic grid spends atoms evenly over all time scales. When
the gaps concentrate in narrow bands, few atoms fall inside them. We therefore add atoms at $Q$
equal-mass quantiles of the training gaps. The order of each such atom is set so that its
coefficient of variation matches the log-distance $d_q$ to its nearest neighbour,
$R_q=\min\{\lceil d_q^{-2}\rceil,4096\}$. Atoms are then sharpest where the data are densest, much
like an adaptive-bandwidth kernel estimate, and the compensator stays exact.

\paragraph{A certified starting point.} The Hawkes backbone of EPT-X can use the dictionary of an
MSX fit, and MSX gives its global optimum with a certificate (Proposition~\ref{prop:gap}).

\begin{proposition}\label{prop:floor}
Initialise EPT-X with the backbone weights $\Theta$ and baselines $\mu$ of the MSX optimum on the
training data, and with the neural read-outs, hazard terms and renewal weights at a scale
$\varepsilon>0$. Then:
\begin{enumerate}
\item as $\varepsilon\to0$, the initial log-likelihood of EPT-X on any data converges to that of the
MSX fit, which is within the certified gap of the best linear phase-type Hawkes process on this
dictionary;
\item if the model is selected by validation likelihood over all training epochs including the
initial one, its validation likelihood is at least that of the initial model.
\end{enumerate}
\end{proposition}

\begin{proof}
\emph{Part 1.} At $\varepsilon=0$ the read-outs, hazards and renewal weights vanish, so EPT-X equals
the backbone (Proposition~\ref{prop:eptx}, part 2), whose intensity and compensator are those of
MSX. The intensities and compensators depend continuously on $\varepsilon$ and are bounded below by
$\mu>0$, so the log-likelihood converges. The second statement is Proposition~\ref{prop:gap}.

\emph{Part 2.} The selected model maximises validation likelihood over a set that contains the
initial model.
\end{proof}

The proposition links the two models of this paper: the certified linear fit is the floor from
which the neural model starts. The next section tests both.

\section{Experiments}\label{sec:experiments}

We test the two models in turn. Synthetic data with known ground truth show what MSX recovers.
Public benchmarks and order-book streams then compare EPT-X with published and re-run baselines,
and ablations show where its gains come from. All configuration choices are made on validation
data, and the test split is scored once.

\subsection{Synthetic recovery}\label{sec:synthetic}
We simulate five ground truths by exact cluster sampling. Several of them lie deliberately
outside the MSX class:
\begin{itemize}
\item S1: a single exponential kernel;
\item S2: a power-law (Lomax) kernel with tail exponent $1.5$;
\item S3: a gamma-shaped delayed peak plus a fast exponential;
\item S4: a sparse four-dimensional network whose kernels have components at 10\,ms and 30\,s;
\item S5: a sparse ten-dimensional network of power-law kernels.
\end{itemize}
Each scenario has ten windows per split and five seeds. We compare MSX with the exponential and
sum-of-exponential maximum-likelihood estimators, the nonparametric EM, ADM4, the conditional-law
method and NPHC. MSX-auto chooses its Erlang order and $\ell_1$ level on held-out training
windows.

Table~\ref{tab:track-a} reports how far each method's held-out log-likelihood falls below that of
the true model. When the truth is a single exponential, the estimators that assume exactly this
form are best. Under a power law, the sum-of-exponential estimators and MSX
perform about equally. When the kernel has structure that no single family captures (a delayed
peak, two separated time scales or a power-law network), MSX-auto has the smallest deficit, 7 to
10 times smaller than the best competitor's. On the two network scenarios it also roughly halves
the best competitor's branching-matrix error and recovers the causal graph perfectly (AUC $=1$,
as do several competitors). On the delayed peak, the nonparametric EM estimates the branching
matrix slightly better (Table~\ref{tab:track-a-full}). Over all 100 MSX fits, the certificate of
Proposition~\ref{prop:gap} was at most $3.3\times10^{-4}$ nats.

\begin{table}[!htbp]\centering\small
\caption{Synthetic recovery: held-out log-likelihood deficit to the true model ($10^{-3}$
nats per event, lower is better; mean $\pm$ s.d.\ over five seeds). Bold: best; underlined:
second best; shaded: our estimators.}\label{tab:track-a}
\setlength{\tabcolsep}{5pt}
\begin{tabular}{lrrrrr}
\toprule
& S1 exp & S2 power law & S3 hump & S4 multiscale & S5 network \\
\midrule
\multicolumn{6}{l}{\textit{Classical estimators}} \\[1pt]
\quad Exp.\ kernel MLE & \textbf{0.05}{\scriptsize\,$\pm$0.05} & 26.05{\scriptsize\,$\pm$0.50} & 101.47{\scriptsize\,$\pm$1.81} & 711.28{\scriptsize\,$\pm$53.09} & 426.80{\scriptsize\,$\pm$368.57} \\
\quad Sum-of-exp.\ MLE & 0.92{\scriptsize\,$\pm$0.13} & \textbf{0.18}{\scriptsize\,$\pm$0.12} & 11.51{\scriptsize\,$\pm$0.46} & -- & -- \\
\quad ADM4 & \textbf{0.05}{\scriptsize\,$\pm$0.05} & 26.05{\scriptsize\,$\pm$0.50} & 101.47{\scriptsize\,$\pm$1.81} & 3.13{\scriptsize\,$\pm$0.25} & 26.51{\scriptsize\,$\pm$0.29} \\
\quad Hawkes EM (nonparam.) & 1.02{\scriptsize\,$\pm$0.07} & 1.73{\scriptsize\,$\pm$0.21} & 2.24{\scriptsize\,$\pm$0.12} & 4.55{\scriptsize\,$\pm$0.29} & 5.23{\scriptsize\,$\pm$0.25} \\
\quad Conditional law & 2.84{\scriptsize\,$\pm$1.03} & 4.91{\scriptsize\,$\pm$0.59} & 10.66{\scriptsize\,$\pm$2.13} & 114.94{\scriptsize\,$\pm$1.20} & 78.64{\scriptsize\,$\pm$4.79} \\
\addlinespace[3pt]\multicolumn{6}{l}{\textit{Ours}} \\[1pt]
\rowcolor{oursbg} \quad MSX ($R=2$) & \underline{0.23}{\scriptsize\,$\pm$0.12} & 0.32{\scriptsize\,$\pm$0.12} & 5.90{\scriptsize\,$\pm$0.36} & \underline{0.54}{\scriptsize\,$\pm$0.12} & 1.19{\scriptsize\,$\pm$0.07} \\
\rowcolor{oursbg} \quad MSX ($R=4$) & 0.28{\scriptsize\,$\pm$0.15} & 0.45{\scriptsize\,$\pm$0.10} & \underline{0.31}{\scriptsize\,$\pm$0.15} & 0.70{\scriptsize\,$\pm$0.11} & 1.63{\scriptsize\,$\pm$0.09} \\
\rowcolor{oursbg} \quad MSX, exponential only & 0.91{\scriptsize\,$\pm$0.13} & \underline{0.19}{\scriptsize\,$\pm$0.11} & 11.51{\scriptsize\,$\pm$0.46} & \textbf{0.32}{\scriptsize\,$\pm$0.08} & \underline{0.82}{\scriptsize\,$\pm$0.08} \\
\rowcolor{oursbg} \quad MSX-auto & 0.24{\scriptsize\,$\pm$0.12} & 0.33{\scriptsize\,$\pm$0.13} & \textbf{0.30}{\scriptsize\,$\pm$0.13} & \textbf{0.32}{\scriptsize\,$\pm$0.08} & \textbf{0.74}{\scriptsize\,$\pm$0.10} \\
\bottomrule
\end{tabular}

\end{table}

\FloatBarrier
\subsection{Public temporal point process benchmarks}\label{sec:public}
We now turn to prediction and compare EPT-X with published neural models. We use the EasyTPP splits of Amazon, Retweet, Taxi, Taobao and StackOverflow and follow the
protocol of \citet{chang2025deep}. The metric is the log-likelihood per scored event, with the
first event of each sequence used only for conditioning, and models are selected on validation
likelihood. Published numbers are taken from \citet{chang2025deep} and \citet{boyd2025hyper}, who
use the same splits and protocol. To check our evaluation code, we re-ran S2P2 on Taxi and obtained
0.524, against 0.522 in the paper. EPT-X reports the exact likelihood.

EPT-X has the best likelihood on three of the five datasets (Table~\ref{tab:track-b}). On Retweet,
the largest dataset with 2.2 million training events, it reaches $\eptxLLRetweet\pm\eptxSDRetweet$
against $-6.348$ for the best published model, with a 25-epoch budget imposed by compute. On Taobao,
starting from the certified MSX fit (Proposition~\ref{prop:floor}) raises EPT-X to
$\eptxLLTaobao\pm\eptxSDTaobao$, against IntensityFree's published $1.318\pm0.017$. On
StackOverflow, the residual mark law gives $\eptxLLStackOverflow\pm\eptxSDStackOverflow$ against
S2P2's $-2.163\pm0.009$, so the two are within one published standard deviation.

On Taxi, EPT-X reaches $\eptxLLTaxi\pm\eptxSDTaxi$ against $0.522\pm0.004$ for S2P2 and HHP, and
its mark term equals that of our S2P2 re-run ($-0.210$ against $-0.211$).

Amazon needs a separate comment. In the EasyTPP release every gap lies in one of two intervals,
$[0.010,0.015]$ (\amzShortShare\% of the test gaps) or $[0.70,0.80]$, and within each interval the
gaps are uniform; Kolmogorov--Smirnov tests do not reject uniformity at the 1\% level on any split.
The timestamps were evidently resampled during preprocessing. A history-free density made of the
two boxes already reaches a time log-likelihood of \amzBoxLL{} nats per event, so on Amazon the
time term mostly measures how sharply a model resolves two hard edges and how well it predicts
which interval comes next. Smooth densities, such as Erlang or log-normal mixtures, lose at the
edges. Data-adaptive atoms and the certified start raise EPT-X to
$\eptxLLAmazon\pm\eptxSDAmazon$, second behind S2P2. We did not add box-shaped components, because
they would fit a preprocessing artefact rather than event dynamics.

The linear MSX model comes close to the neural models on Taobao and Retweet, whose dynamics are
largely self-exciting.

\paragraph{Tuning budgets.} We selected the EPT-X configuration on validation data from a search
over Erlang orders, grid size, encoder depth and width, an attention encoder, batch size, weight
decay and the three new components, with one seed per variant (Table~\ref{tab:ablation} and
Appendix~\ref{app:details}). The published baselines, in contrast, carry their authors' own tuning,
and our order-book baselines use the authors' EasyTPP configurations with one seed each. The
margins on Retweet, Taobao and the order books are large compared with any seed-to-seed variation
we observed.

\begin{table}[!htbp]\centering\small
\caption{Public benchmarks: test log-likelihood per event (higher is better). Published
numbers are from \citet{chang2025deep} and \citet{boyd2025hyper}. Our rows give the mean $\pm$ s.d.\ over
seeds (EPT-X: \eptxSeeds; other rows one seed), and the EPT-X configuration is selected per
dataset on validation. Bold: best; underlined: second best; shaded: our models.}
\label{tab:track-b}
\setlength{\tabcolsep}{5pt}
\begin{tabular}{lrrrrr}
\toprule
& Amazon & Retweet & Taxi & Taobao & StackOverflow \\
\midrule
\multicolumn{6}{l}{\textit{Neural TPPs, published (\pubsrc)}} \\[1pt]
\quad RMTPP & $-$2.136 & $-$7.098 & 0.346 & 1.003 & $-$2.480 \\
\quad SAHP & $-$2.074 & $-$6.708 & 0.298 & 1.168 & $-$2.341 \\
\quad THP & $-$2.096 & $-$6.659 & 0.372 & 0.790 & $-$2.338 \\
\quad IntensityFree & 0.496 & $-$10.344 & 0.453 & \underline{1.318} & $-$2.233 \\
\quad MHP & $-$2.091 & $-$6.564 & 0.370 & 0.636 & $-$2.346 \\
\quad NHP & 0.129 & $-$6.348 & 0.514 & 1.157 & $-$2.241 \\
\quad AttNHP & 0.484 & $-$6.499 & 0.493 & 1.259 & $-$2.194 \\
\quad S2P2 & \textbf{0.781} & $-$6.365 & \textbf{0.522} & 1.304 & \underline{$-$2.163} \\
\quad HHP & 0.603 & $-$6.357 & \textbf{0.522} & 1.273 & $-$2.195 \\
\addlinespace[3pt]\multicolumn{6}{l}{\textit{Classical Hawkes, our runs}} \\[1pt]
\quad ADM4 & $-$2.260 & $-$6.832 & $-$0.315 & 1.184 & $-$2.473 \\
\quad Conditional law & $-$4.469 & $-$8.317 & $-$2.922 & $-$63.919 & $-$16.140 \\
\quad Hawkes EM (nonparam.) & $-$1.002 & $-$6.655 & $-$0.308 & 1.269 & $-$2.370 \\
\addlinespace[3pt]\multicolumn{6}{l}{\textit{Ours}} \\[1pt]
\rowcolor{oursbg} \quad MSX ($R=2$) & $-$1.943 & $-$6.569 & $-$0.214 & 1.299 & $-$2.292 \\
\rowcolor{oursbg} \quad MSX-auto & $-$1.927 & $-$6.568 & $-$0.193 & 1.305 & $-$2.292 \\
\rowcolor{oursbg} \quad EPT & $-$1.203 & \underline{$-$5.731} & 0.437 & 1.311 & $-$2.186 \\
\rowcolor{oursbg} \quad EPT-X & \underline{0.723}{\scriptsize\,$\pm$0.002} & \textbf{$-$5.676}{\scriptsize\,$\pm$0.003} & \underline{0.517}{\scriptsize\,$\pm$0.002} & \textbf{1.361}{\scriptsize\,$\pm$0.004} & \textbf{$-$2.162}{\scriptsize\,$\pm$0.001} \\
\bottomrule
\end{tabular}

\end{table}

\paragraph{Next-event prediction.} Table~\ref{tab:pred} reports the two prediction metrics used by
\citet{chang2025deep}. For the time of the next event, EPT-X has the lowest RMSE on Retweet and
StackOverflow, ties the best published value to three decimals on Taxi and Taobao, and is second on
Amazon (0.327 against 0.324 for HHP). Mark accuracies are reported in the lower panel.

\begin{table}[!htbp]\centering\small
\caption{Next-event prediction on the public benchmarks: time RMSE (top, lower is better; Retweet
in seconds) and mark accuracy (bottom, \%, higher is better). Published numbers from
\citet{chang2025deep} and \citet{boyd2025hyper}, Table~2(b,c) of each. EPT-X: mean $\pm$ s.d.\ over seeds of the
validation-selected configuration. Bold: best; underlined: second best; shaded: ours.}\label{tab:pred}
\setlength{\tabcolsep}{5pt}
\begin{tabular}{lrrrrr}
\toprule
& Amazon & Retweet & Taxi & Taobao & StackOverflow \\
\midrule
\multicolumn{6}{l}{\textit{Published (\pubsrc)}} \\[1pt]
\quad RMTPP & 0.338 & 16488 & 0.283 & \underline{0.126} & 1.049 \\
\quad SAHP & 0.335 & 16102 & 0.290 & \underline{0.126} & 1.031 \\
\quad THP & 0.332 & 16268 & 0.285 & \textbf{0.125} & 1.033 \\
\quad IntensityFree & \underline{0.327} & 16625 & 0.362 & \textbf{0.125} & 1.340 \\
\quad MHP & 0.329 & 16109 & 0.284 & \underline{0.126} & 1.046 \\
\quad NHP & 0.339 & 15911 & \underline{0.282} & \underline{0.126} & 1.019 \\
\quad AttNHP & 2.656 & 16171 & 1.739 & 0.130 & 1.256 \\
\quad S2P2 & \underline{0.327} & 15987 & \textbf{0.281} & \underline{0.126} & \underline{1.014} \\
\quad HHP & \textbf{0.324} & \underline{15589} & \textbf{0.281} & \underline{0.126} & 1.016 \\
\addlinespace[3pt]\multicolumn{6}{l}{\textit{Ours}} \\[1pt]
\rowcolor{oursbg} \quad EPT-X & \underline{0.327}{\scriptsize\,$\pm$0.000} & \textbf{14402}{\scriptsize\,$\pm$32} & \textbf{0.281}{\scriptsize\,$\pm$0.000} & \textbf{0.125}{\scriptsize\,$\pm$0.000} & \textbf{1.011}{\scriptsize\,$\pm$0.000} \\
\bottomrule
\end{tabular}
\\[6pt]
\begin{tabular}{lrrrrr}
\toprule
& Amazon & Retweet & Taxi & Taobao & StackOverflow \\
\midrule
\multicolumn{6}{l}{\textit{Published (\pubsrc)}} \\[1pt]
\quad RMTPP & 30.8 & 53.4 & 91.4 & 60.9 & 45.6 \\
\quad SAHP & 32.4 & 57.5 & 91.4 & 60.5 & 44.7 \\
\quad THP & 34.6 & 60.2 & 91.4 & 60.0 & 46.6 \\
\quad IntensityFree & 35.9 & 50.4 & 91.8 & 61.0 & 45.6 \\
\quad MHP & 35.1 & 60.0 & 91.4 & 60.7 & 46.5 \\
\quad NHP & 39.4 & \textbf{61.4} & 92.9 & \textbf{61.5} & 47.1 \\
\quad AttNHP & 38.9 & 60.7 & 92.6 & 61.3 & \textbf{48.2} \\
\quad S2P2 & \underline{40.7} & \underline{61.3} & \textbf{93.1} & 61.1 & \underline{47.5} \\
\quad HHP & \textbf{40.9} & \underline{61.3} & \underline{93.0} & \underline{61.4} & 47.1 \\
\addlinespace[3pt]\multicolumn{6}{l}{\textit{Ours}} \\[1pt]
\rowcolor{oursbg} \quad EPT-X & 35.9{\scriptsize\,$\pm$0.0} & 58.3{\scriptsize\,$\pm$0.1} & 91.7{\scriptsize\,$\pm$0.1} & \underline{61.4}{\scriptsize\,$\pm$0.1} & 45.4{\scriptsize\,$\pm$0.1} \\
\bottomrule
\end{tabular}

\end{table}

\FloatBarrier
\subsection{Limit-order-book event streams}\label{sec:lob}
Order-book streams are a harder regime than the public benchmarks: events arrive in bursts down to
microsecond gaps, across six event types. We use the public LOBSTER \citep{huang2011lobster} samples of AAPL, AMZN, GOOG, INTC and MSFT
from 21 June 2012, with six event types, a chronological 60/20/20 split and sequences of 128
events. EasyTPP rejects tied timestamps, so we displace tied messages by a uniform offset in
$(0,100\,\text{ns})$, and every model is trained and scored on these same files. The neural
baselines run in EasyTPP with the S2P2 authors' configurations, with batches of 32 to 64 instead
of 256 to fit an 8\,GB GPU. Each is trained with a 100-epoch budget, and the strongest ones also
with longer budgets. We report each at the budget with the best validation likelihood, so the
longer budgets can only help them.

EPT-X is the best model on every stock (Table~\ref{tab:track-c}). It is ahead of MSX-auto, EPT and
all seven neural baselines, and ahead of the best neural baseline by \trackCmarginLo{} to
\trackCmarginHi{} nats per event.

One baseline needs a correction. EasyTPP's IntensityFree replaces every gap below $10^{-5}$\,s by
$10^{-5}$\,s before evaluating its log-normal mixture, which can then put an arbitrarily narrow
component on that value. The result is not the likelihood of the observed gaps. On INTC and MSFT,
29\% and 32\% of the test gaps are this short (10 to 14\% on the other stocks), and there the
clamped model reports the highest numbers. Retrained and scored on the actual gaps with the same
1500-epoch budget, it loses 1.1 and 0.7 nats per event and falls below EPT-X, on INTC also below
MSX. For scale, a denser MSX dictionary or sharper EPT-X atoms change the INTC validation
likelihood by at most about 0.05. We report both versions, rank only the corrected one, and include
the patch in the \codename.

\begin{table}[!htbp]\centering\small
\caption{LOBSTER order books: test log-likelihood per event (higher is better; mean $\pm$
s.d.\ over seeds where several were run). Bold: best; underlined: second best; shaded: our models.
$^\ast$As shipped, IntensityFree clamps gaps at $10^{-5}$\,s before scoring, so its score is not
the likelihood of the observed data; it is greyed and not ranked. The IntensityFree row is
retrained and scored on the actual gaps.}\label{tab:track-c}
\setlength{\tabcolsep}{5pt}
\begin{tabular}{lrrrrr}
\toprule
& AAPL & AMZN & GOOG & INTC & MSFT \\
\midrule
\multicolumn{6}{l}{\textit{Classical Hawkes}} \\[1pt]
\quad Exp.\ kernel MLE & 1.961 & -- & $-$0.311 & -- & -- \\
\quad Hawkes EM (nonparam.) & 3.370 & 2.243 & 2.851 & 5.771 & 6.258 \\
\quad ADM4 & 1.961 & 1.067 & 1.306 & 4.787 & 5.211 \\
\quad Conditional law & $-$1.018 & $-$17.814 & $-$32.784 & 0.379 & $-$1.270 \\
\addlinespace[3pt]\multicolumn{6}{l}{\textit{Neural TPPs (EasyTPP, our runs)}} \\[1pt]
\quad RMTPP & 0.243 & $-$0.686 & $-$0.948 & 2.882 & 3.285 \\
\quad THP & $-$0.081 & $-$0.948 & $-$1.075 & 1.714 & 1.855 \\
\quad SAHP & 1.777 & 0.462 & $-$0.428 & 3.915 & 4.557 \\
\quad NHP & 1.374 & 0.842 & 0.846 & 3.874 & 4.183 \\
\quad AttNHP & 1.220 & $-$0.168 & 0.004 & 3.497 & 3.771 \\
\quad S2P2 & 2.315 & 0.717 & 0.990 & 4.949 & 5.416 \\
\quad IntensityFree & 3.108 & 2.103 & 2.726 & 5.713 & 6.377 \\
\quad \textcolor{gray}{IntensityFree$^\ast$ (as shipped)} & \textcolor{gray}{3.260} & \textcolor{gray}{1.921} & \textcolor{gray}{2.679} & \textcolor{gray}{6.790} & \textcolor{gray}{7.070} \\
\addlinespace[3pt]\multicolumn{6}{l}{\textit{Ours}} \\[1pt]
\rowcolor{oursbg} \quad MSX ($R=2$) & 3.326 & 2.249 & 2.870 & 5.812 & 6.307 \\
\rowcolor{oursbg} \quad MSX-auto & 3.361 & \underline{2.281} & \underline{2.883} & 5.814 & 6.310 \\
\rowcolor{oursbg} \quad EPT & \underline{3.380}{\scriptsize\,$\pm$0.008} & 2.229{\scriptsize\,$\pm$0.003} & 2.806{\scriptsize\,$\pm$0.005} & \underline{6.024}{\scriptsize\,$\pm$0.007} & \underline{6.530}{\scriptsize\,$\pm$0.005} \\
\rowcolor{oursbg} \quad EPT-X & \textbf{3.534}{\scriptsize\,$\pm$0.016} & \textbf{2.332}{\scriptsize\,$\pm$0.005} & \textbf{2.992}{\scriptsize\,$\pm$0.012} & \textbf{6.163}{\scriptsize\,$\pm$0.002} & \textbf{6.666}{\scriptsize\,$\pm$0.001} \\
\bottomrule
\end{tabular}

\end{table}

\FloatBarrier
\subsection{Ablations}\label{sec:ablation}
The results above use one selected configuration per dataset. Table~\ref{tab:ablation} shows what
each component contributes on validation data.

The two structural parts matter everywhere. Adding the renewal channel improves EPT on every
dataset, and without the Hawkes backbone EPT reaches only \eptNoBackbone{} nats per event on
LOBSTER AAPL, below the classical MSX estimator. The encoder, in contrast, is not the bottleneck.
A causal Transformer encoder \citep{vaswani2017attention} in place of the GRU lowers the validation
likelihood on all five datasets, and a deeper encoder (three residual layers with LayerNorm and
dropout), a wider one (128 units) and S2P2's larger batch bring no gain.

The remaining components are selected where the data call for them. Sharp atoms help where the gap
distribution has sharp peaks: by 1.1 nats per event on Amazon, where a finer grid helps further,
and by 0.05 on INTC. Data-adaptive atoms add to this on Amazon, whose gaps fill two narrow
intervals. The residual mark law helps where the marks follow the elapsed time closely, on Taxi
and StackOverflow. The certified warm start helps most on Taobao and Amazon, and on Retweet it
raises the validation likelihood by 0.03 nats per event. Elsewhere the default setting of each
component is selected.

\paragraph{Where the events come from.} Because every channel has a closed-form compensator, the
expected number of events splits exactly by channel (Table~\ref{tab:attribution}). The split differs
across data. On Taxi, four in five events come from the renewal channel, and on Retweet four in five
come from the neural state. On Amazon, whose gaps are drawn from two fixed intervals, the renewal
channel accounts for all but 0.4\% of them, as a renewal process should. On Taobao and StackOverflow, about half come from the Hawkes backbone,
and on the order books the renewal channel carries 41 to 58\%, with the backbone and the neural state
sharing the rest. The backbone is not constrained to be stationary, so its weights
describe excitation within a sequence rather than a long-run branching ratio.

\begin{table}[!htbp]\centering\small
\caption{Exact share (\%) of the expected number of test events contributed by each channel of the
selected EPT-X model (seed 0).}\label{tab:attribution}
\begin{tabular}{lrrrrr}
\toprule
Data & Hawkes backbone & renewal & baseline & hazard terms & neural state \\
\midrule
Amazon & 0.4 & 99.6 & 0.0 & 0.0 & 0.0 \\
Retweet & 9.5 & 8.2 & 0.9 & 1.0 & 80.4 \\
Taxi & 13.2 & 80.9 & 0.2 & 5.3 & 0.4 \\
Taobao & 57.0 & 36.0 & 6.0 & 0.9 & 0.1 \\
StackOverflow & 49.4 & 3.0 & 0.5 & 44.9 & 2.2 \\
LOBSTER AAPL & 30.6 & 40.6 & 1.5 & 1.5 & 25.8 \\
LOBSTER AMZN & 34.8 & 46.4 & 1.1 & 3.2 & 14.5 \\
LOBSTER GOOG & 32.7 & 43.5 & 0.6 & 2.4 & 20.7 \\
LOBSTER INTC & 9.3 & 53.6 & 1.1 & 1.6 & 34.3 \\
LOBSTER MSFT & 8.7 & 58.3 & 0.7 & 0.3 & 32.0 \\
\bottomrule
\end{tabular}

\end{table}

\paragraph{Cost.} Exactness does not have to cost time (Table~\ref{tab:cost}). Evaluated with the
parallel scan of Lemma~\ref{lem:scan}, an EPT-X epoch, including validation, is \scanSpeedTaxi{}
times faster than with the event loop on Taxi and \scanSpeedAAPL{} times faster on LOBSTER AAPL,
where it takes \scanSecAAPL{}\,s against 5.8\,s for S2P2 on the same GPU. The two evaluations give
the same likelihood and gradients up to floating-point rounding.

\begin{table}[!htbp]\centering\small
\caption{Training cost per epoch (including validation) and model size. EPT-X: one epoch of the
same model timed with the event loop and with the parallel scan. EasyTPP baselines: average over
their full epoch budget.}\label{tab:cost}
\begin{tabular}{llllrr}
\toprule
Data & Model & Compensator & Device & s / epoch & Parameters \\
\midrule
\rowcolor{oursbg} Taxi & EPT-X, event loop & exact & GPU & 9.1 & 146k \\
\rowcolor{oursbg}  & EPT-X, parallel scan & exact & GPU & 3.2 & 146k \\
 & S2P2 & Monte Carlo & GPU & 7.7 & 35k \\
\midrule
\rowcolor{oursbg} LOBSTER AAPL & EPT-X, event loop & exact & GPU & 13.0 & 76k \\
\rowcolor{oursbg}  & EPT-X, parallel scan & exact & GPU & 4.5 & 76k \\
 & S2P2 & Monte Carlo & GPU & 5.8 & 34k \\
 & NHP & Monte Carlo & GPU & 1.6 & 232k \\
 & AttNHP & Monte Carlo & GPU & 7.5 & 20k \\
 & THP & Monte Carlo & GPU & 0.4 & 51k \\
 & IntensityFree & exact & GPU & 0.4 & 7k \\
\bottomrule
\end{tabular}

\end{table}

\begin{table}[!htbp]\centering\footnotesize
\caption{EPT-X ablations: validation log-likelihood per event (higher is better). The EPT and
EPT-X rows are the final runs (mean $\pm$ s.d.\ over seeds); each variant is one seed. Encoder and
atom variants use a 100-epoch budget. The last three rows add one new component each, with the
final budget on Taxi, Taobao and StackOverflow and 100 epochs on Amazon and INTC, where they are
added to the sharp-atom configuration. Bold: best; underlined: second best.}\label{tab:ablation}
\setlength{\tabcolsep}{4pt}
\begin{tabular}{lrrrrr}
\toprule
& Taxi & Taobao & StackOverflow & Amazon & INTC \\
\midrule
\multicolumn{6}{l}{\textit{Validation log-likelihood per event}} \\[1pt]
\quad EPT (no renewal channel) & 0.392 & 1.200 & $-$2.186 & $-$1.193 & 5.360{\scriptsize\,$\pm$0.007} \\
\quad EPT-X & \underline{0.466}{\scriptsize\,$\pm$0.004} & \underline{1.205}{\scriptsize\,$\pm$0.002} & \underline{$-$2.161}{\scriptsize\,$\pm$0.003} & $-$0.487{\scriptsize\,$\pm$0.000} & 5.463{\scriptsize\,$\pm$0.003} \\
\quad \quad + sharp atoms (orders to 1024) & 0.457 & 1.194 & $-$2.173 & 0.604 & \textbf{5.516} \\
\quad \quad + sharp atoms, 48 scales & 0.455 & 1.177 & $-$2.182 & 0.643 & 5.511 \\
\quad \quad + deep encoder (3 layers) & 0.396 & 1.142 & $-$2.264 & $-$0.507 & 5.461 \\
\quad \quad + sharp atoms + deep encoder & 0.398 & 1.152 & $-$2.248 & 0.579 & 5.496 \\
\quad \quad + wide encoder (128 units) & 0.431 & 1.188 & $-$2.195 & $-$0.489 & 5.459 \\
\quad \quad + attention encoder (causal Transformer) & 0.415 & 1.169 & $-$2.238 & $-$0.531 & 5.416 \\
\quad \quad batch 256 (S2P2 setting) & 0.440 & 1.189 & $-$2.442 & $-$0.518 & 5.435 \\
\quad \quad + residual mark law & \textbf{0.477} & 1.198 & \textbf{$-$2.156} & 0.647 & \underline{5.512} \\
\quad \quad + data-adaptive atoms & 0.463 & 1.179 & $-$2.171 & \textbf{0.677} & 5.504 \\
\quad \quad + certified warm start & 0.457 & \textbf{1.243} & $-$2.162 & \underline{0.659} & 5.479 \\
\bottomrule
\end{tabular}

\end{table}

\FloatBarrier

\section{Application: how fast do markets absorb macroeconomic news?}\label{sec:application}

We now return to the measurement question that motivated MSX and apply it to market data around
scheduled macroeconomic releases. Figure~\ref{fig:pipeline} summarises the steps from raw data to
absorption times.

\begin{figure}[!htbp]\centering
\begin{tikzpicture}[
    font=\small,
    >={Stealth[length=2.2mm]},
    box/.style={draw, thick, minimum height=9mm, minimum width=20mm, align=center, inner sep=3pt},
    fitbox/.style={box, fill=black!6},
    res/.style={align=center, font=\footnotesize},
    node distance=6mm,
  ]
  \node[box] (tk)  {ticks and\\ releases};
  \node[box, right=of tk] (ev) {$\delta$-crossing\\ events};
  \node[box, right=of ev] (win) {release and\\ placebo windows};
  \node[fitbox, right=of win] (fit) {MSX fit\\ EM, Newton};
  \node[box, right=of fit] (emb) {Markov\\ embedding};

  \draw[->] (tk) -- (ev);
  \draw[->] (ev) -- (win);
  \draw[->] (win) -- (fit);
  \draw[->] (fit) -- node[font=\footnotesize, above] {$\hat{\bm\theta}$} (emb);

  \node[res, below=6mm of fit] (gap) {duality gap,\\ placebo test};
  \draw[->] (fit) -- (gap);
  \node[res, below=6mm of emb] (abs) {$\bm m(t)$, $t_{50}$, $t_{90}$,\\ drift};
  \draw[->] (emb) -- (abs);
\end{tikzpicture}
\caption{Measurement pipeline. Prices become $\delta$-crossing events, which are grouped into
release windows and matched placebo windows. A single certified MSX fit gives the duality gap and
the placebo test, and its Markov embedding gives the response $\bm m(t)$, the absorption times and
the expected price drift in closed form.}\label{fig:pipeline}
\end{figure}
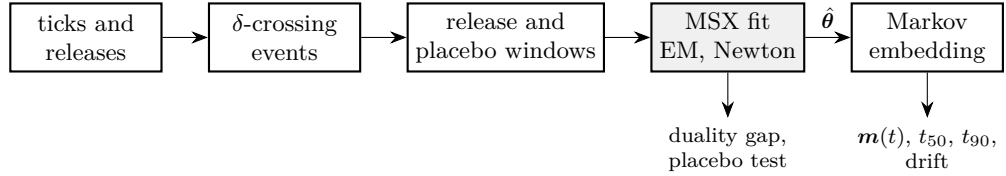

\paragraph{Releases.} Our calendar covers 670 scheduled releases between January 2022 and August
2026. It contains CPI, nonfarm payrolls, PPI, retail sales, the GDP advance estimate and initial
jobless claims, each dated by its first vintage in ALFRED, together with FOMC statements and press
conferences, ECB decisions and press conferences, and Bank of Japan decisions. The surprise of a
release is the difference between its first print and a real-time expectation, standardised and
split into positive and negative parts. As expectations we use the Cleveland Fed inflation nowcast
for CPI, the Atlanta Fed GDPNow forecast for GDP, and a statistical forecast from earlier first
prints for the other releases.

\paragraph{Prices and events.} We use millisecond quotes for EUR/USD, USD/JPY, gold and an
S\&P~500 contract for difference, and Binance futures trades for bitcoin and ether. Each price
series is turned into events by $\delta$-crossing: an upward (downward) event is recorded whenever
the log price has risen (fallen) by $\delta$ since the previous event. We choose $\delta$ per asset,
between 0.25 and 3 basis points, so that the median rate is about 0.2 events per second. Each asset
thus contributes two event types, and the twelve types together form the dimensions of the model.
Windows run from 30 minutes before a release to 60 minutes after it, preceded by 30 minutes of
history that is used for the states but not scored.

\paragraph{Placebo windows.} Each release window is paired with a placebo window at the same clock
time on the same weekday, one to six weeks away, chosen so that no modelled release and none of
sixteen second-tier US releases falls within three hours. Placebo windows contain pseudo-releases
of their own type, so the placebo test is part of the likelihood rather than a separate
regression. The pooled model has \nWindows{} windows (\nNewsWindows{} release and
\nPlaceboWindows{} placebo) with \nEventsScored{} million scored events, and its certificate bounds
the distance to the optimum by \gapMaxPooled{} nats per dimension.

\paragraph{Endogeneity.} The fitted model first tells us how endogenous these markets are. The
spectral radius of the branching matrix is $\rho(G)=\rhoPooled$, so
roughly seven in every eight price events are triggered by earlier price events rather than
directly by news. Fitted year by year, $\rho(G)$ declines from 2022 to 2026 (\rhoByYear), and the
slowest echo mode decays in about \relaxPooled{} minutes.

\paragraph{Falsification.} Before reading the news kernels, we test them. Under the null
hypothesis of no effect, a nonnegative news kernel still
has a positive expected mass, because estimates cannot fall below zero. A comparison with
release-free placebo windows is therefore essential. With the $\ell_1$ level chosen by held-out
likelihood, \nPassPlacebo{} of the \nKinds{} release types exceed their placebo by at least 50\%
(Figure~\ref{fig:placebo}): CPI (\placeboCpi$\times$), payrolls (\placeboNfp$\times$), PPI
(\placeboPpi$\times$), the FOMC statement (\placeboFomc$\times$) and retail sales
(\placeboRetail$\times$). Jobless claims, the ECB and BoJ decisions and the press conferences show
no direct effect beyond their placebos. Two earlier specifications failed this test, and we report
them in Appendix~\ref{app:history}.

\begin{figure}[!htbp]\centering
\includegraphics[width=0.72\linewidth]{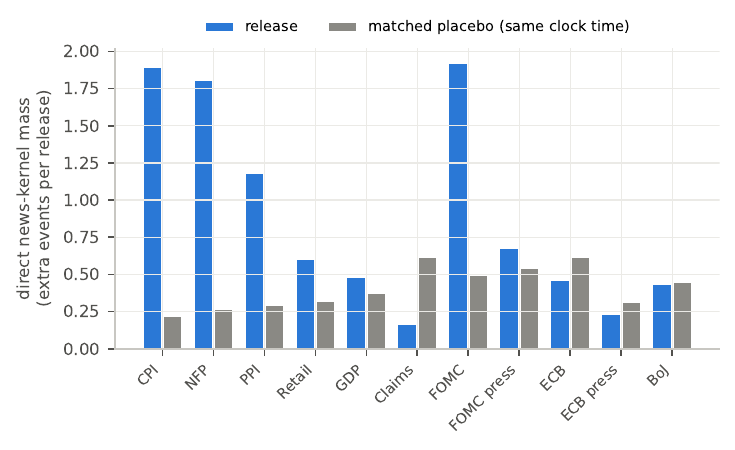}
\caption{Placebo test. Direct news-kernel mass, the expected number of extra events per release
summed over all assets, for each release type (blue) and for its matched placebo at the same clock
time on a release-free day (grey). A type passes if its mass exceeds the placebo mass by at least
50\%.}\label{fig:placebo}
\end{figure}

\paragraph{Echo-corrected absorption.} For the releases that pass the placebo test, the direct
news kernel is absorbed almost immediately: its median $t_{50}$ is \tDirMed{}\,s. The echoes
stretch the response considerably. Half of the extra activity arrives within \tEchoLo{} to
\tEchoHi{}\,s, and ninety percent within \tNinetyLo{} to \tNinetyHi{} minutes
(Figure~\ref{fig:absorption} and Table~\ref{tab:absorption}). The ``news half-life'' $\ln 2/\beta$
therefore understates absorption time by one to two orders of magnitude. Because positive and
negative surprises enter the news kernels separately, the model also gives expected price paths
after a release (Figure~\ref{fig:drift}). A higher-than-expected CPI print lowers bitcoin and
ether, and a lower one lifts them.

\begin{figure}[!htbp]\centering
\includegraphics[width=\linewidth]{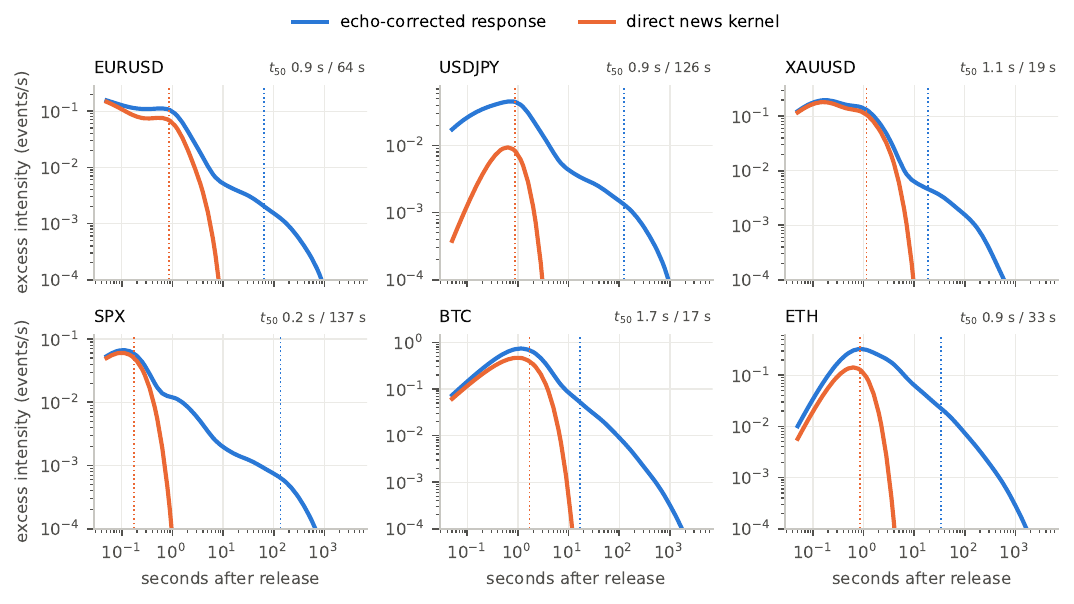}
\caption{Response of each asset to a CPI release, in excess events per second (both axes
logarithmic). Orange: direct news kernel. Blue: echo-corrected response $\bm m(t)$, which adds the
events triggered by earlier events. Dotted lines and panel titles give $t_{50}$, direct / echo-corrected.}\label{fig:absorption}
\end{figure}

\begin{table}[!htbp]\centering\small
\caption{Echo-corrected absorption times from the pooled model: $t_{50}$ in seconds / $t_{90}$
in minutes after the release. $^\ast$Release types whose news-kernel mass does not exceed their
placebo by 50\% (Figure~\ref{fig:placebo}).}\label{tab:absorption}
\setlength{\tabcolsep}{4pt}
\begin{tabular}{lcccccc}
\toprule
Release & EURUSD & USDJPY & Gold & S\&P 500 & BTC & ETH \\
\midrule
CPI & 64 / 9.1 & 126 / 11.0 & 19 / 4.8 & 137 / 11.0 & 17 / 6.1 & 33 / 7.8 \\
NFP & 94 / 10.1 & 109 / 10.7 & 16 / 5.2 & 80 / 8.8 & 27 / 7.5 & 24 / 7.3 \\
PPI & 120 / 10.8 & 94 / 10.3 & 40 / 6.4 & 79 / 8.6 & 23 / 7.2 & 39 / 8.7 \\
FOMC & 89 / 10.1 & 76 / 9.8 & 82 / 9.0 & 77 / 8.9 & 41 / 9.2 & 36 / 9.1 \\
Retail & 59 / 9.1 & 77 / 9.8 & 87 / 9.2 & 114 / 10.4 & 26 / 7.6 & 26 / 7.8 \\
\midrule
GDP$^\ast$ & 95 / 10.0 & 81 / 9.6 & 26 / 5.1 & 145 / 11.2 & 27 / 7.3 & 27 / 7.5 \\
FOMC press conf.$^\ast$ & 88 / 10.0 & 111 / 10.9 & 60 / 8.1 & 59 / 8.0 & 45 / 9.5 & 52 / 10.0 \\
Claims$^\ast$ & 124 / 11.0 & 56 / 9.1 & 87 / 9.3 & 81 / 9.1 & 36 / 9.0 & 40 / 9.4 \\
ECB$^\ast$ & 61 / 9.2 & 73 / 9.6 & 88 / 9.0 & 123 / 10.8 & 37 / 9.2 & 42 / 9.6 \\
ECB press conf.$^\ast$ & 137 / 11.3 & 90 / 10.2 & 38 / 6.5 & 63 / 8.2 & 25 / 7.3 & 32 / 8.0 \\
BoJ$^\ast$ & 93 / 10.2 & 80 / 9.9 & 45 / 7.3 & 73 / 8.8 & 35 / 8.5 & 37 / 8.8 \\
\bottomrule
\end{tabular}

\end{table}

\begin{figure}[!htbp]\centering
\includegraphics[width=\linewidth]{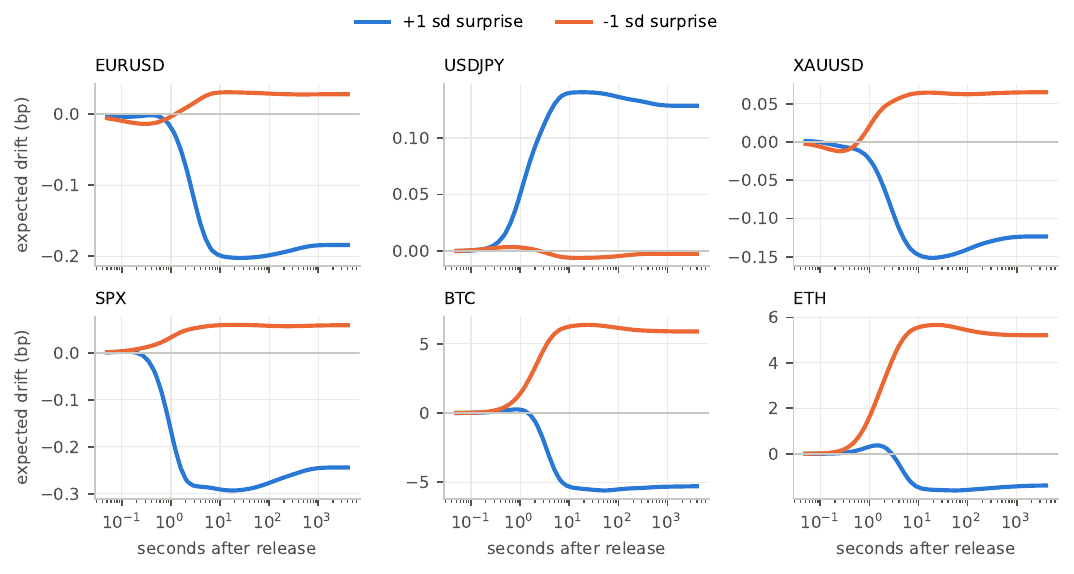}
\caption{Expected cumulative price change, in basis points, after a CPI surprise of $+1$
(blue) or $-1$ (orange) standard deviation, from the pooled model. Time axis logarithmic.}\label{fig:drift}
\end{figure}

\paragraph{Goodness of fit.} These conclusions rest on the fit of the model, which we check with
time-rescaled residuals \citep{brown2002time}. They have
Kolmogorov--Smirnov statistics of at most \ksMax{} against Exp(1); the per-window
excess-dispersion test rejects in \edLo--\edHi\% of windows, compared with 42\% (power-law
kernel) and 93\% (double exponential) of trading days reported for EBS FX data by
\citet{rambaldi2015modeling} with the same test.

\paragraph{Forecasting.} Out of sample, the closed-form MSX forecasts of post-release activity
improve on climatology at every horizon, and its drift forecasts are more informative than the
best baseline (Appendix~\ref{app:forecast}).

\section{Discussion and limitations}\label{sec:discussion}

Writing every kernel as a nonnegative combination of phase-type densities makes Hawkes estimation
a convex problem whose solution can be certified, and it turns the response to a shock into a
matrix exponential. Adding a neural history encoder and a mixture-of-Erlang renewal channel to the
same state space gives a model whose likelihood is still exact, which trains with a parallel scan
and which can start from the certified linear fit. This model is competitive with the strongest
neural temporal point processes, and better than them on the order books, Retweet and Taobao.

We draw three conclusions. First, the speed at which news is absorbed is not a property of the
news kernel alone. The market's own echoes lengthen it by one to two orders of magnitude, and the
correction can be computed in closed form. Second, falsification tests are worth building in. Two
of our earlier specifications produced plausible news effects that the placebo test showed to be
artefacts. Third, benchmark comparisons should check that every model is scored on the same
observed data and that the data carry the structure the benchmark is meant to test. A single clamp
inside one implementation was enough to reverse a ranking, and the Amazon gaps are drawn from two
fixed intervals.

\paragraph{Limitations.} For payrolls, claims, PPI and retail sales our expectations are
statistical forecasts rather than survey consensus, and the FX data have a gap in the first half
of 2023. The order-book evaluation uses the public LOBSTER samples, which cover one trading day per
stock. On the order books, the neural baselines were trained with smaller batches than their
authors used, to fit an 8\,GB GPU. On the public benchmarks we compare with published numbers and
have checked our evaluation with one re-run. On Retweet, EPT-X was trained for 25 epochs.

\subsubsection*{Broader Impact Statement}
This work models market event streams with public or licensable data and makes no use of
personal data. Better models of how quickly prices absorb public information can inform market
monitoring and the timing of disclosures, and exact likelihoods make benchmark comparisons more
reliable. Faster or more accurate event forecasts could also be used in trading; the methods
described are standard statistical tools and do not by themselves confer an informational
advantage. We release code and patches so that the reported numbers, including the corrected
baseline, can be reproduced.

\subsubsection*{Reproducibility Statement}
All result tables and every number quoted in the text are generated by scripts from the logged
outputs of the experiments; the \codename{} contains the models, the experiment and
table scripts, the EasyTPP patches (tie handling and the IntensityFree clamp), configuration
files and unit tests (exact compensators against numerical quadrature, causality of the
encoders, nesting of EPT-X in EPT). Public datasets are used through their standard splits;
the market data are licensed from their providers and are not redistributed. The code is
available at \url{https://github.com/itz-sayak/infodiff}.

\bibliography{refs}
\bibliographystyle{tmlr}

\appendix
\section{Proofs}\label{app:proofs}

\begin{proof}[Proof of Proposition~\ref{prop:gap}]
We use the notation of the proposition: $X$ has rows $\bm x_n^\top=\bm x(t_n)^\top$,
$\lambda_n=\bm x_n^\top\bm\theta$, and
$\ell(\bm\theta)=\sum_{n=1}^{N}\log\lambda_n-(\bm\Phi+\bm\gamma)^\top\bm\theta$.

\emph{Step 1 (an upper bound for every feasible point).} For $u,v>0$, Fenchel's inequality gives
\[
\log u\;\le\;v\,u-1-\log v .
\]
Take any vector $\bm v=(v_1,\dots,v_N)$ with positive entries such that
$X^\top\bm v\le\bm\Phi+\bm\gamma$ componentwise. Applying the inequality with $u=\lambda_n$ and
$v=v_n$ and summing over $n$,
\[
\sum_{n}\log\lambda_n\;\le\;\sum_n v_n\lambda_n-N-\sum_n\log v_n
\;=\;\bm\theta^\top X^\top\bm v-N-\sum_n\log v_n .
\]
Since $\bm\theta\ge0$ and $X^\top\bm v\le\bm\Phi+\bm\gamma$, we have
$\bm\theta^\top X^\top\bm v\le\bm\theta^\top(\bm\Phi+\bm\gamma)$. Subtracting
$(\bm\Phi+\bm\gamma)^\top\bm\theta$ from both sides therefore gives
\[
\ell(\bm\theta)\;\le\;D(\bm v):=-N-\sum_n\log v_n .
\]
This holds for every feasible $\bm\theta$, in particular for the maximiser, so
$\ell^\star\le D(\bm v)$.

\emph{Step 2 (a computable choice of $\bm v$).} Let $\bm\theta$ be the current fit and set
$v_n=1/(s\lambda_n)$, with $s=\max_p\bigl(\sum_n x_{np}/\lambda_n\bigr)/(\Phi_p+\gamma_p)$. Then
$(X^\top\bm v)_p=\frac{1}{s}\sum_n x_{np}/\lambda_n\le\Phi_p+\gamma_p$ by the definition of $s$, so
$\bm v$ is admissible in Step 1. For this choice,
\[
D(\bm v)=-N+N\log s+\sum_n\log\lambda_n ,
\qquad\text{hence}\qquad
\ell^\star-\ell(\bm\theta)\;\le\;D(\bm v)-\ell(\bm\theta)=N\log s+\bm\theta^\top(\bm\Phi+\bm\gamma)-N .
\]

\emph{Step 3 (the bound vanishes at the optimum).} The Karush--Kuhn--Tucker conditions of the
maximisation state that $\sum_n x_{np}/\lambda_n=\Phi_p+\gamma_p$ when $\theta_p>0$ and
$\sum_n x_{np}/\lambda_n\le\Phi_p+\gamma_p$ when $\theta_p=0$. Hence $s=1$, so $N\log s=0$. Moreover,
\[
\bm\theta^\top(\bm\Phi+\bm\gamma)=\sum_p\theta_p\sum_n\frac{x_{np}}{\lambda_n}
=\sum_n\frac{\bm x_n^\top\bm\theta}{\lambda_n}=N ,
\]
so the second term is also zero.
\end{proof}

\section{Specification history of the application}\label{app:history}
We report the specifications that failed the built-in falsification test. The final
specification was chosen because it passed this test, not because of the results it produced.

\paragraph{Version 1.} This version used news kernels with time scales up to 30 minutes and Erlang
order 3, a constant baseline in each window and time-of-day effects in UTC. It also allowed tied
events: a trade that crossed several $\delta$ levels produced several events with the same
timestamp. The single placebo type absorbed \vOnePlacebo{} extra events per pseudo-release, which
is comparable to real releases, and the residuals of the crypto series had Kolmogorov--Smirnov
statistics up to \vOneKSmax. Two causes explain this. News kernels longer than the post-release
window absorbed slow trends within the window, and tied events violate the assumption that events
occur one at a time.

\paragraph{Version 2.} This version kept one event per timestamp, with the number of crossed levels
as a mark, restricted news kernels to time scales between 0.1 and 300\,s, used a linear baseline
within each window and modelled seasonality in New York time. The placebo mass fell to
\vTwoPlacebo{} but was still significant. Placebo mass appeared at every release time of day, so
unmodelled 08:30 releases could not be the whole explanation.

\paragraph{Version 3 (final).} Placebo days now also avoid sixteen second-tier US releases. The
news kernels carry an $\ell_1$ penalty, weighted by exposure, whose level is chosen by the
likelihood of held-out windows (Table~\ref{tab:l1}). Without the penalty, nonnegative estimates
inflate release and placebo masses equally. The held-out likelihood peaks at
$\ell_1=\lOneSelected$, which is also where release and placebo masses separate.

\begin{table}[!htb]\centering\small
\caption{Choice of the $\ell_1$ penalty on the news kernels. Held-out log-likelihood on the
2024 windows relative to the best level (bold), and the mean news-kernel mass of release and
placebo windows.}\label{tab:l1}
\begin{tabular}{cccc}
\toprule
$\ell_1$ level & held-out LL (rel.) & mean release mass & mean placebo mass \\
\midrule
0.2 & {$-$93.7} & 10.79 & 6.23 \\
0.4 & {$-$27.1} & 5.20 & 2.89 \\
0.8 & \textbf{0.0} & 1.81 & 1.03 \\
1.6 & {$-$20.5} & 0.44 & 0.37 \\
\bottomrule
\end{tabular}

\end{table}

\section{Forecasting the post-release burst}\label{app:forecast}
On the \nTestReleases{} releases of 2025--2026, an MSX model fitted on 2022--2024 forecasts
the number of $\delta$-crossings in the first 1, 5 and 15 minutes in closed form
(Table~\ref{tab:track-d}). It improves on same-type climatology at every horizon; the
\bestActivityModel{} and HAR regressions, which are fitted directly to this target, score lower
errors on activity. For the direction of the move the structural model is more informative: the correlation between its
forecast and the realised one-minute drift is \driftCorrOursOne{}, against \driftCorrBaseOne{}
for the best baseline (mean drift by release type and surprise sign).

\begin{table}[!htbp]\centering\small
\caption{Post-release activity forecasts on the \nTestReleases{} test releases of
2025--2026, from models fitted on 2022--2024. $N$ is the number of $\delta$-crossings after the
release; QLIKE scores the forecast of realised variance. Mean over assets; lower is better. Bold:
best; underlined: second best; shaded: our model.}\label{tab:track-d}
\setlength{\tabcolsep}{5pt}
\begin{tabular}{lrrrrrr}
\toprule
& \multicolumn{3}{c}{RMSE of $\log(1+N)$} & \multicolumn{3}{c}{QLIKE of realised variance} \\
\cmidrule(lr){2-4}\cmidrule(lr){5-7}
Method & 1 min & 5 min & 15 min & 1 min & 5 min & 15 min \\
\midrule
Climatology & 1.440 & 1.308 & 1.199 & \textbf{1.263} & 0.916 & 0.701 \\
Event-study OLS & \textbf{1.140} & \textbf{0.925} & \textbf{0.776} & 2.552 & \textbf{0.580} & \textbf{0.372} \\
HAR & \underline{1.217} & \underline{0.985} & \underline{0.820} & \underline{1.467} & \underline{0.762} & \underline{0.468} \\
\rowcolor{oursbg} MSX closed form & 1.322 & 1.049 & 0.955 & 2.913 & 1.036 & 0.518 \\
\bottomrule
\end{tabular}

\end{table}

\section{Synthetic recovery: all metrics}\label{app:synthetic}
Table~\ref{tab:track-a-full} reports every metric for the five scenarios of Section~\ref{sec:synthetic}.

\begin{table}[!tp]\centering\footnotesize
\caption{Synthetic recovery, all metrics (mean $\pm$ s.d.\ over five seeds):
relative error of the branching matrix and of its spectral radius, kernel $L^1$ error and held-out
deficit dLL ($10^{-3}$ nats/event), all lower is better, and edge AUC (higher is better). AUC needs
absent true edges, so it is defined only for S4 and S5 (n/a for the complete graphs S1--S3).
``--'': not produced by the method; NPHC estimates only the branching matrix. Missing rows are fits
that did not converge. Bold: best; underlined: second best within each scenario; shaded: ours.}\label{tab:track-a-full}
\scriptsize\setlength{\tabcolsep}{5pt}\renewcommand{\arraystretch}{1.0}
\begin{tabular}{lrrrrr}
\toprule
& $G$ error & $\rho$ error & kernel $L^1$ & dLL ($10^{-3}$) & edge AUC \\
\midrule
\multicolumn{6}{l}{\textit{S1: exponential}} \\[1pt]
\quad Exp.\ kernel MLE & \textbf{0.018}{\scriptsize\,$\pm$0.007} & \textbf{0.005}{\scriptsize\,$\pm$0.003} & \textbf{0.018}{\scriptsize\,$\pm$0.008} & \textbf{0.05}{\scriptsize\,$\pm$0.05} & n/a \\
\quad Sum-of-exp.\ MLE & 0.065{\scriptsize\,$\pm$0.022} & 0.030{\scriptsize\,$\pm$0.013} & 0.165{\scriptsize\,$\pm$0.023} & 0.92{\scriptsize\,$\pm$0.13} & n/a \\
\quad ADM4 & \underline{0.018}{\scriptsize\,$\pm$0.007} & \underline{0.005}{\scriptsize\,$\pm$0.003} & \underline{0.018}{\scriptsize\,$\pm$0.008} & \underline{0.05}{\scriptsize\,$\pm$0.05} & n/a \\
\quad Hawkes EM (nonparam.) & 0.035{\scriptsize\,$\pm$0.015} & 0.017{\scriptsize\,$\pm$0.007} & 0.173{\scriptsize\,$\pm$0.010} & 1.02{\scriptsize\,$\pm$0.07} & n/a \\
\quad NPHC (cumulants) & 0.247{\scriptsize\,$\pm$0.087} & 0.032{\scriptsize\,$\pm$0.024} & -- & -- & n/a \\
\quad Conditional law & 0.120{\scriptsize\,$\pm$0.054} & 0.017{\scriptsize\,$\pm$0.016} & 0.193{\scriptsize\,$\pm$0.022} & 2.84{\scriptsize\,$\pm$1.03} & n/a \\
\rowcolor{oursbg} \quad MSX-auto & 0.054{\scriptsize\,$\pm$0.032} & 0.020{\scriptsize\,$\pm$0.018} & 0.093{\scriptsize\,$\pm$0.035} & 0.24{\scriptsize\,$\pm$0.12} & n/a \\
\rowcolor{oursbg} \quad MSX ($R=2$) & 0.056{\scriptsize\,$\pm$0.032} & 0.019{\scriptsize\,$\pm$0.018} & 0.092{\scriptsize\,$\pm$0.035} & 0.23{\scriptsize\,$\pm$0.12} & n/a \\
\rowcolor{oursbg} \quad MSX ($R=4$) & 0.056{\scriptsize\,$\pm$0.035} & 0.023{\scriptsize\,$\pm$0.018} & 0.105{\scriptsize\,$\pm$0.034} & 0.28{\scriptsize\,$\pm$0.15} & n/a \\
\rowcolor{oursbg} \quad MSX, exponential only & 0.055{\scriptsize\,$\pm$0.025} & 0.026{\scriptsize\,$\pm$0.015} & 0.157{\scriptsize\,$\pm$0.022} & 0.91{\scriptsize\,$\pm$0.13} & n/a \\
\addlinespace[3pt]\multicolumn{6}{l}{\textit{S2: power law}} \\[1pt]
\quad Exp.\ kernel MLE & 0.299{\scriptsize\,$\pm$0.011} & 0.219{\scriptsize\,$\pm$0.007} & 0.465{\scriptsize\,$\pm$0.003} & 26.05{\scriptsize\,$\pm$0.50} & n/a \\
\quad Sum-of-exp.\ MLE & \underline{0.057}{\scriptsize\,$\pm$0.012} & 0.016{\scriptsize\,$\pm$0.007} & \textbf{0.094}{\scriptsize\,$\pm$0.011} & \textbf{0.18}{\scriptsize\,$\pm$0.12} & n/a \\
\quad ADM4 & 0.299{\scriptsize\,$\pm$0.011} & 0.219{\scriptsize\,$\pm$0.007} & 0.465{\scriptsize\,$\pm$0.003} & 26.05{\scriptsize\,$\pm$0.50} & n/a \\
\quad Hawkes EM (nonparam.) & 0.057{\scriptsize\,$\pm$0.014} & 0.037{\scriptsize\,$\pm$0.008} & 0.188{\scriptsize\,$\pm$0.006} & 1.73{\scriptsize\,$\pm$0.21} & n/a \\
\quad NPHC (cumulants) & 0.153{\scriptsize\,$\pm$0.154} & 0.041{\scriptsize\,$\pm$0.047} & -- & -- & n/a \\
\quad Conditional law & 0.102{\scriptsize\,$\pm$0.067} & 0.029{\scriptsize\,$\pm$0.019} & 0.291{\scriptsize\,$\pm$0.028} & 4.91{\scriptsize\,$\pm$0.59} & n/a \\
\rowcolor{oursbg} \quad MSX-auto & 0.064{\scriptsize\,$\pm$0.029} & 0.037{\scriptsize\,$\pm$0.029} & 0.108{\scriptsize\,$\pm$0.022} & 0.33{\scriptsize\,$\pm$0.13} & n/a \\
\rowcolor{oursbg} \quad MSX ($R=2$) & 0.057{\scriptsize\,$\pm$0.021} & \underline{0.015}{\scriptsize\,$\pm$0.010} & 0.114{\scriptsize\,$\pm$0.015} & 0.32{\scriptsize\,$\pm$0.12} & n/a \\
\rowcolor{oursbg} \quad MSX ($R=4$) & 0.062{\scriptsize\,$\pm$0.014} & 0.016{\scriptsize\,$\pm$0.006} & 0.131{\scriptsize\,$\pm$0.017} & 0.45{\scriptsize\,$\pm$0.10} & n/a \\
\rowcolor{oursbg} \quad MSX, exponential only & \textbf{0.054}{\scriptsize\,$\pm$0.022} & \textbf{0.014}{\scriptsize\,$\pm$0.010} & \underline{0.097}{\scriptsize\,$\pm$0.011} & \underline{0.19}{\scriptsize\,$\pm$0.11} & n/a \\
\addlinespace[3pt]\multicolumn{6}{l}{\textit{S3: hump}} \\[1pt]
\quad Exp.\ kernel MLE & 0.166{\scriptsize\,$\pm$0.008} & 0.053{\scriptsize\,$\pm$0.004} & 0.801{\scriptsize\,$\pm$0.004} & 101.47{\scriptsize\,$\pm$1.81} & n/a \\
\quad Sum-of-exp.\ MLE & 0.038{\scriptsize\,$\pm$0.011} & 0.018{\scriptsize\,$\pm$0.005} & 0.419{\scriptsize\,$\pm$0.005} & 11.51{\scriptsize\,$\pm$0.46} & n/a \\
\quad ADM4 & 0.166{\scriptsize\,$\pm$0.008} & 0.053{\scriptsize\,$\pm$0.004} & 0.801{\scriptsize\,$\pm$0.004} & 101.47{\scriptsize\,$\pm$1.81} & n/a \\
\quad Hawkes EM (nonparam.) & \textbf{0.024}{\scriptsize\,$\pm$0.011} & \textbf{0.010}{\scriptsize\,$\pm$0.004} & 0.187{\scriptsize\,$\pm$0.008} & 2.24{\scriptsize\,$\pm$0.12} & n/a \\
\quad NPHC (cumulants) & 0.240{\scriptsize\,$\pm$0.096} & 0.045{\scriptsize\,$\pm$0.048} & -- & -- & n/a \\
\quad Conditional law & 0.149{\scriptsize\,$\pm$0.062} & 0.017{\scriptsize\,$\pm$0.010} & 0.319{\scriptsize\,$\pm$0.050} & 10.66{\scriptsize\,$\pm$2.13} & n/a \\
\rowcolor{oursbg} \quad MSX-auto & \underline{0.032}{\scriptsize\,$\pm$0.016} & 0.016{\scriptsize\,$\pm$0.007} & \underline{0.080}{\scriptsize\,$\pm$0.013} & \textbf{0.30}{\scriptsize\,$\pm$0.13} & n/a \\
\rowcolor{oursbg} \quad MSX ($R=2$) & 0.069{\scriptsize\,$\pm$0.009} & 0.037{\scriptsize\,$\pm$0.004} & 0.347{\scriptsize\,$\pm$0.007} & 5.90{\scriptsize\,$\pm$0.36} & n/a \\
\rowcolor{oursbg} \quad MSX ($R=4$) & 0.038{\scriptsize\,$\pm$0.010} & 0.017{\scriptsize\,$\pm$0.005} & \textbf{0.074}{\scriptsize\,$\pm$0.010} & \underline{0.31}{\scriptsize\,$\pm$0.15} & n/a \\
\rowcolor{oursbg} \quad MSX, exponential only & 0.034{\scriptsize\,$\pm$0.009} & \underline{0.015}{\scriptsize\,$\pm$0.004} & 0.416{\scriptsize\,$\pm$0.008} & 11.51{\scriptsize\,$\pm$0.46} & n/a \\
\addlinespace[3pt]\multicolumn{6}{l}{\textit{S4: multiscale network}} \\[1pt]
\quad Exp.\ kernel MLE & 0.173{\scriptsize\,$\pm$0.041} & 0.095{\scriptsize\,$\pm$0.025} & 1.635{\scriptsize\,$\pm$0.040} & 711.28{\scriptsize\,$\pm$53.09} & 1.000{\scriptsize\,$\pm$0.000} \\
\quad ADM4 & 0.399{\scriptsize\,$\pm$0.003} & 0.255{\scriptsize\,$\pm$0.001} & 0.410{\scriptsize\,$\pm$0.001} & 3.13{\scriptsize\,$\pm$0.25} & 1.000{\scriptsize\,$\pm$0.000} \\
\quad Hawkes EM (nonparam.) & 0.153{\scriptsize\,$\pm$0.012} & 0.045{\scriptsize\,$\pm$0.010} & 0.358{\scriptsize\,$\pm$0.012} & 4.55{\scriptsize\,$\pm$0.29} & 1.000{\scriptsize\,$\pm$0.000} \\
\quad NPHC (cumulants) & 1.030{\scriptsize\,$\pm$0.304} & 0.211{\scriptsize\,$\pm$0.186} & -- & -- & 0.588{\scriptsize\,$\pm$0.155} \\
\quad Conditional law & 0.464{\scriptsize\,$\pm$0.286} & 0.139{\scriptsize\,$\pm$0.050} & 1.059{\scriptsize\,$\pm$0.074} & 114.94{\scriptsize\,$\pm$1.20} & 0.831{\scriptsize\,$\pm$0.087} \\
\rowcolor{oursbg} \quad MSX-auto & \textbf{0.084}{\scriptsize\,$\pm$0.014} & \underline{0.016}{\scriptsize\,$\pm$0.006} & \textbf{0.128}{\scriptsize\,$\pm$0.026} & \textbf{0.32}{\scriptsize\,$\pm$0.08} & 1.000{\scriptsize\,$\pm$0.000} \\
\rowcolor{oursbg} \quad MSX ($R=2$) & 0.093{\scriptsize\,$\pm$0.015} & 0.021{\scriptsize\,$\pm$0.005} & 0.173{\scriptsize\,$\pm$0.024} & 0.54{\scriptsize\,$\pm$0.12} & 1.000{\scriptsize\,$\pm$0.000} \\
\rowcolor{oursbg} \quad MSX ($R=4$) & 0.101{\scriptsize\,$\pm$0.019} & 0.027{\scriptsize\,$\pm$0.007} & 0.207{\scriptsize\,$\pm$0.035} & 0.70{\scriptsize\,$\pm$0.11} & \underline{1.000}{\scriptsize\,$\pm$0.000} \\
\rowcolor{oursbg} \quad MSX, exponential only & \underline{0.084}{\scriptsize\,$\pm$0.014} & \textbf{0.016}{\scriptsize\,$\pm$0.006} & \underline{0.128}{\scriptsize\,$\pm$0.026} & \underline{0.32}{\scriptsize\,$\pm$0.08} & \textbf{1.000}{\scriptsize\,$\pm$0.000} \\
\addlinespace[3pt]\multicolumn{6}{l}{\textit{S5: 10-dim.\ power-law network}} \\[1pt]
\quad Exp.\ kernel MLE & 0.337{\scriptsize\,$\pm$0.153} & 0.103{\scriptsize\,$\pm$0.032} & 1.302{\scriptsize\,$\pm$0.897} & 426.80{\scriptsize\,$\pm$368.57} & 0.977{\scriptsize\,$\pm$0.050} \\
\quad ADM4 & 0.167{\scriptsize\,$\pm$0.003} & 0.096{\scriptsize\,$\pm$0.003} & 0.319{\scriptsize\,$\pm$0.002} & 26.51{\scriptsize\,$\pm$0.29} & 1.000{\scriptsize\,$\pm$0.000} \\
\quad Hawkes EM (nonparam.) & \underline{0.085}{\scriptsize\,$\pm$0.008} & 0.045{\scriptsize\,$\pm$0.008} & 0.280{\scriptsize\,$\pm$0.010} & 5.23{\scriptsize\,$\pm$0.25} & 1.000{\scriptsize\,$\pm$0.000} \\
\quad NPHC (cumulants) & 1.348{\scriptsize\,$\pm$0.477} & 0.520{\scriptsize\,$\pm$0.372} & -- & -- & 0.672{\scriptsize\,$\pm$0.042} \\
\quad Conditional law & 0.461{\scriptsize\,$\pm$0.047} & 0.037{\scriptsize\,$\pm$0.026} & 0.641{\scriptsize\,$\pm$0.025} & 78.64{\scriptsize\,$\pm$4.79} & 0.950{\scriptsize\,$\pm$0.033} \\
\rowcolor{oursbg} \quad MSX-auto & \textbf{0.045}{\scriptsize\,$\pm$0.004} & \textbf{0.015}{\scriptsize\,$\pm$0.005} & \textbf{0.079}{\scriptsize\,$\pm$0.005} & \textbf{0.74}{\scriptsize\,$\pm$0.10} & 1.000{\scriptsize\,$\pm$0.000} \\
\rowcolor{oursbg} \quad MSX ($R=2$) & 0.111{\scriptsize\,$\pm$0.013} & 0.040{\scriptsize\,$\pm$0.013} & 0.180{\scriptsize\,$\pm$0.017} & 1.19{\scriptsize\,$\pm$0.07} & 1.000{\scriptsize\,$\pm$0.000} \\
\rowcolor{oursbg} \quad MSX ($R=4$) & 0.122{\scriptsize\,$\pm$0.008} & 0.051{\scriptsize\,$\pm$0.013} & 0.218{\scriptsize\,$\pm$0.015} & 1.63{\scriptsize\,$\pm$0.09} & \underline{1.000}{\scriptsize\,$\pm$0.000} \\
\rowcolor{oursbg} \quad MSX, exponential only & 0.099{\scriptsize\,$\pm$0.014} & \underline{0.030}{\scriptsize\,$\pm$0.014} & \underline{0.145}{\scriptsize\,$\pm$0.017} & \underline{0.82}{\scriptsize\,$\pm$0.08} & \textbf{1.000}{\scriptsize\,$\pm$0.000} \\
\bottomrule
\end{tabular}

\end{table}

\section{Experimental details}\label{app:details}

\paragraph{MSX in the application.}
\begin{itemize}
\item Endogenous kernels: 10 rates from 5\,ms to 600\,s, Erlang order 2.
\item News kernels: 8 rates from 0.1 to 300\,s, Erlang order 3, marks $(1,z^+,z^-)$.
\item Anticipation kernels: 4 rates from 10 to 300\,s.
\item Baseline: a linear trend within each window plus 96 time-of-day hat functions.
\end{itemize}

\paragraph{MSX on the benchmarks.} MSX uses one baseline shared by all sequences and a grid of 12
rates from the 1st percentile of the training gaps to five times their 99.9th percentile.
MSX-auto chooses the Erlang order from $\{1,2,3,4\}$ and the $\ell_1$ level from
$\{0,0.01,0.05\}$ on the validation split.

\paragraph{EPT and EPT-X.}
\begin{itemize}
\item Encoder: a GRU with 64 units. Latent state: 4 channels of 8 learnable rates. The number of
Erlang phases was chosen on validation from $\{2,4\}$, together with the hidden size from
$\{64,128\}$, on Taxi, Taobao and StackOverflow.
\item Hawkes backbone: 10 rates with 2 phases each.
\item Optimisation: AdamW \citep{loshchilov2019decoupled} with learning rate $10^{-2}$, 1\% linear
warm-up and cosine decay, gradient clipping at 1, and batches of 64 that are halved automatically
if the GPU runs out of memory. Training stops when the validation likelihood has not improved for
40 epochs, within a budget of 300 epochs (25 epochs with batches of 32 on Retweet).
\item Renewal channel: 24 mean gaps spaced logarithmically from half the 0.1\% quantile of the
training gaps to five times their 99.9\% quantile, with orders $\{1,4,16\}$. The mark
distributions start at the empirical mark frequencies.
\item Selected on validation, the LOBSTER streams and Amazon add orders up to 1024 and a
history-dependent shift of the atoms. Amazon also uses a finer grid: the single-seed validation
log-likelihood was 0.604, 0.643, 0.653 and 0.649 with 24, 48, 96 and 192 scales, so we use 96.
\item Data-adaptive atoms: 48 quantile atoms (96 on Amazon and INTC), orders capped at 4096.
\item Residual mark law: one hidden layer of 64 GELU units on the history summary and the
standardised log-gap features; its output layer starts at zero.
\item Certified start: the MSX-auto dictionary and fit of the same dataset; read-outs, hazard
weights and renewal weights start at $e^{-7}$ and are divided by 10 until the initial validation
likelihood is within 0.01 nats per event of MSX. On Retweet the read-outs keep the initial scale
$e^{-7}$.
\item Selected configurations: residual mark law with dropout 0.1 on Taxi; residual mark law on
StackOverflow; MSX start on Taobao and Retweet; data-adaptive atoms, residual mark law and
certified start on Amazon; sharp atoms on the order books.
\end{itemize}

\paragraph{Baselines.} The neural baselines run in EasyTPP with the authors' configuration files.
The classical Hawkes estimators use the \texttt{tick} library \citep{bacry2018tick}, with
positive-constrained accelerated gradient descent, tolerance $10^{-10}$ and decay grids matched to
MSX. A short patch, included in the \codename, makes the IntensityFree clamp
configurable.

\paragraph{Compute.} All experiments ran on one laptop with an 8\,GB GPU and a 28-core CPU, with
each CPU job limited to four threads. Earlier EPT-X runs used the event loop, which runs about as
fast on the CPU as on the GPU; runs with the new components used the parallel scan, mostly on the
GPU. Both evaluations give the same likelihood.

\end{document}